%% file: neurips_2025.tex
\documentclass{article}

\usepackage[preprint]{neurips_2025}

\usepackage{natbib} 
\usepackage[utf8]{inputenc} 
\usepackage[T1]{fontenc}    
\usepackage{hyperref}       
\usepackage{url}            
\usepackage{booktabs}       
\usepackage{amsfonts}       
\usepackage{nicefrac}       
\usepackage{microtype}      
\usepackage{xcolor}  
\usepackage{algorithm} 
\usepackage{algorithmic} 
\usepackage{enumitem} 

\title{Learning Unknown Interdependencies for Decentralized Root Cause Analysis in
Nonlinear Dynamical Systems}

\author{
  Ayush Mohanty\thanks{School of Industrial and Systems Engineering, Georgia Institute of Technology, Atlanta, GA 30332}
  \And
  Paritosh Ramanan\thanks{School of Industrial Engineering and Management, Oklahoma State University, Stillwater, OK 74708}
  \And
  Nagi Gebraeel\footnotemark[1]
}

\usepackage{amsmath}
\usepackage{amssymb}
\usepackage{mathtools}
\usepackage{amsthm}
\usepackage{subfiles}
\usepackage{subcaption}
\usepackage{booktabs}
\theoremstyle{plain}
\newtheorem{theorem}{Theorem}[section]
\newtheorem{proposition}[theorem]{Proposition}
\newtheorem{lemma}[theorem]{Lemma}

\theoremstyle{definition}
\newtheorem{definition}[theorem]{Definition}
\newtheorem{assumption}[theorem]{Assumption}
\newtheorem{remark}[theorem]{Remark}

\usepackage[textsize=tiny]{todonotes}
\usepackage[toc,page,header]{appendix}
\usepackage{minitoc}
\begin{document}

\maketitle

\begin{abstract}

\subfile{Abstract.tex}
\end{abstract}
\section{Introduction}\label{sec:Intro}
\subfile{Introduction.tex}
\section{Related Work}\label{sec:RelatedWork}
\subfile{RelatedWork.tex}
\section{Problem Setting}\label{sec:ProblemSetting}
\subfile{ProblemSetting.tex}
\section{Training: Learning Cross-Client Interdependencies}\label{sec:training}
\subfile{Training.tex}
\section{Inference: Root Cause Analysis}\label{section:inferencing}
\subfile{Inference.tex}

\section{Convergence Analysis}\label{sec:Convergence}
\subfile{Convergence.tex}
\section{Privacy Analysis}\label{sec:privacy}
\subfile{PrivacyAnalysis.tex}
\section{Experiments}\label{sec:Experiments}
\subfile{Experiments.tex}
\section{Conclusion and Limitations}\label{sec:conclusionAndlimitations}
\subfile{ConclusionANDLimitations.tex}




\clearpage
\bibliography{reference}






\newpage
\appendices
\tableofcontents
\newpage
\addtocontents{toc}{\protect\setcounter{tocdepth}{3}}
\subfile{Appendix/Appendix}
\end{document}

%% file: Abstract.tex
Root cause analysis (RCA) in networked industrial systems, such as supply chains and power networks, is notoriously difficult due to unknown and dynamically evolving interdependencies among geographically distributed clients. These clients represent heterogeneous physical processes and industrial assets equipped with sensors that generate large volumes of \textit{nonlinear}, high-dimensional, and heterogeneous IoT data. Classical RCA methods require partial or full knowledge of the system’s dependency graph, which is rarely available in these complex networks. While federated learning (FL) offers a natural framework for decentralized settings, most existing FL methods assume homogeneous feature spaces and retrainable client models.  These assumptions are not compatible with our problem setting.  Different clients have different data features and often run fixed, proprietary models that cannot be modified. This paper presents a federated cross-client interdependency learning methodology for feature-partitioned, nonlinear time-series data, without requiring access to raw sensor streams or modifying proprietary client models. Each proprietary local client model is augmented with a Machine Learning (ML) model that encodes cross-client interdependencies. These ML models are coordinated via a global server that enforces representation consistency while preserving privacy through calibrated differential privacy noise. RCA is performed using model residuals and anomaly flags. We establish theoretical convergence guarantees and validate our approach on extensive simulations and a real-world industrial cybersecurity dataset.

%% file: Introduction.tex
Root-cause analysis (RCA) is a fundamental diagnostic method used to identify the underlying causes of system-level failures. RCA becomes significantly challenging in modern networked systems that are characterized by complex, and \textbf{often unknown}, interdependencies among geographically dispersed industrial assets  (hereafter referred to as clients), such as smart cities \cite{taoLocatingCompromisedData2018}, power networks \cite{falahatiSmartGrid2012, caiCascadingFailureAnalysis2016}, and supply chains \cite{yeMachineFeedstockInterdependence2022}. Characterizing these cross-client interdependencies is critical to mitigating the impacts of any disruptions and identifying their root causes.

Several conventional methods have been developed for learning interdependencies. They include Granger causality \cite{granger1980testing, tankNeuralGrangerCausality2022a}, vector autoregression \cite{pmlr-v37-geiger15, pmlr-v48-melnyk16}, and structural equation modeling \cite{kline2023principles}.  However, these methods assume that data from all the clients can be aggregated into a central server. In our industrial setting, this assumption breaks down. Clients generate large volumes of high-dimensional sensor data and are often governed by stringent data-sovereignty regulations. Therefore, centralized approaches for learning interdependencies are not adequate.

\textbf{Limitations of existing decentralized methods.} Federated learning (FL) \cite{mcmahan2017communication} offers a promising alternative by enabling model training directly at the data source, thereby preserving privacy and minimizing communication overhead. In principle, FL provides a framework well-suited for distributed RCA: clients retain sensitive operational data locally while contributing to a shared model that can be used for RCA. However, there are unique challenges that limit the applicability of existing methods. 

\begin{enumerate}[leftmargin=*]
    \item First, existing FL approaches that aim to capture directed dependencies (causality) ~\cite{mian2023nothing,fan2023score} are generally developed under the assumption of a homogeneous, shared feature space across clients (i.e., horizontal FL). This assumption is not applicable to our problem setting, where clients correspond to distinct processes or equipment types, each with unique sensors and feature representations. In such feature-partitioned (i.e., vertical FL) environments, cross-client causal modeling becomes substantially more challenging due to heterogeneous data features and site-specific temporal dynamics. 
    \item Second, our problem involves highly complex latent system-wide interactions that require discovering unknown, possibly causal, interdependencies directly from decentralized data \cite{zhengModelingAnalysisCascading2024, buldyrev2010catastrophic}. Unfortunately, federated graph learning formulations ~\cite{wan2024federated, chen2021fedgraph} do not support this capability. For example, federated graph neural networks (FedGNNs) such as ~\cite{he2021fedgraphnn, mai2023vertical, meng2021cross} rely on the availability of an explicit or partially observable adjacency matrix to encode the graph structure. This adjacency matrix typically reflects pre-defined relationships among clients, which are generally unknown in our problem setting.

    \item Third, in many industrial settings, each client operates a proprietary model supplied by its original equipment manufacturer (OEM), which cannot be modified, retrained, or directly updated. Most state-of-the-art FL algorithms assume that client models are trainable and participate in federated parameter updates. However, this assumption fails in scenarios where model internals are inaccessible due to intellectual property constraints or certification requirements. Thus, learning must occur without altering the underlying client models. This poses a fundamental departure from prevailing FL paradigms.
\end{enumerate}

\textbf{Main Contributions:}
This paper proposes a federated framework that relies on collaboratively learning interdependencies across feature-partitioned OT clients. The key contributions are:
\begin{itemize}[leftmargin=*]
     \item We develop a novel federated cross-client interdependency learning algorithm that utilizes non-linear local ML models at each client to encode interdependency information communicated by an ML-based global server model. The local ML models are used to \textbf{\textit{augment}} the representation generated by proprietary client models.  
    \item Our framework enables decentralized RCA, where the server identifies the source of anomalies without direct access to client states or raw data. We propose that clients utilize the \text{\textit{residuals}} from proprietary client models and their augmented forms to generate anomaly flags that identify the root cause (RC) and distinguish it from the propagated effect (PE).
    \item We provide theoretical guarantees that, under standard assumptions, our framework converges to the performance of a fully \textit{centralized oracle}. Furthermore, to protect sensitive client data, we incorporate differential privacy into training and inference, where states and anomaly flags are perturbed with controlled noise, enabling secure communication.
    \item Extensive evaluations on synthetic datasets confirm the framework's effectiveness in learning cross-client interdependencies, anomaly detection, and RCA performance. We also study the sensitivity to detection threshold and scalability in large systems. We further validate it on a real-world industrial testbed, demonstrating its effectiveness in learning interdependencies and performing decentralized RCA in an industrial control systems dataset. 
\end{itemize}

%% file: RelatedWork.tex
\textbf{Anomaly Detection \& Root Cause Analysis (RCA) in time-series:} Multivariate time-series anomaly detection has been well-studied in centralized settings, addressing variable correlations \cite{Schmidl2022, Foo2022}. Decentralized approaches \cite{Liu2021, Huong2021, Cui2022} improve scalability but fail to model interdependencies. Most anomaly detection methods omit RCA, limiting their mitigation utility. RCA typically uses centralized methods such as Granger causality \cite{granger1980testing, tankNeuralGrangerCausality2022a} and Vector Autoregression \cite{lutkepohl2005new}, which rely on raw data access. Recent RCA studies explore causal discovery and graphical models \cite{pham2024root, assaad2023root, 7563819, budhathokiCausalStructurebasedRoot2022, liCausalInferenceBasedRoot2022, wangIncrementalCausalGraph2023, shahRootCauseDetection2018}, but a few such as \cite{liuDatadrivenRootcauseAnalysis2017, liuRootcauseAnalysisTimeseries2021} address decentralization. Integrating decentralized RCA with anomaly detection, especially under unknown interdependencies, remains unexplored.

\textbf{Vertical Federated Learning (VFL):} 
Our approach aligns with VFL, where clients hold different features of the same sample; \textbf{in our case, features are time-series measurements, and samples correspond to timestamps}. Traditional VFL frameworks focus on classification tasks with fixed labels, which are either shared across clients \cite{FDML, Huang2022, Castiglia2022}, held by one client (called the ``active client/party'') \cite{SecureBoost, AFSGD, VFBSGD}, or stored on servers \cite{Gu2023, Castiglia2023, Li2023}. With few exceptions, such as \cite{VAFL, Ma2023}, most studies use global models as simple aggregators rather than leveraging ML-based models, preventing them from learning interdependencies.

\textbf{Related Domains:} \text{\textit{Federated Multi Task Learning }}\cite{FedMTL, FedMDMTL, FedGraphMTL, FedMSplit} focuses on parameter sharing in horizontally partitioned datasets, overlooking feature-based interdependency modeling.  \text{\textit{Split Learning}} \cite{singh2019detailed, poirot2019split, vepakomma2018split} handles feature-based partitioning by splitting neural network layers but does not address anomaly detection or interdependencies.  

%% file: ProblemSetting.tex
We consider a decentralized setup with $M$ clients. Each client \( m \) observes \textbf{high-dimensional data} \( y_m^t \in \mathbb{R}^{D_m} \) and uses a \textit{proprietary client model} to compute its \textbf{lower-dimensional states} \( {(x_m^t)}_c \in \mathbb{R}^{P_m} \) (\( D_m \gg P_m \)). Since these proprietary client models rely solely on locally available data \( y_m^t \), it limits their ability to accurately capture cross-client interdependencies. To address this limitation, we use iterative optimization  during \underline{training}:
\begin{enumerate}[leftmargin=*]
    \item \textbf{Client.} Each client \text{\textit{augments}} \( {(x_m^t)}_c \) with a \textit{non-linear ML} \( \phi_m(y_m^t; \theta_m) \), producing \( {(x_m^t)}_a \) where \( \theta_m \) \textit{locally encodes interdependencies}. Clients minimize a loss \( (L_m)_a \) using both local (\( \nabla_{\theta_m}(L_m)_a \)) and server-provided (\( \nabla_{\theta_m}L_s \)) gradients. It then updates ML function \( \phi(y_m^t; \theta_m) \) and communicates states -- \( {(x_m^t)}_c, {(x_m^t)}_a \) to the server.
    \item \textbf{Server.} The server processes \( {(x_m^t)}_c, {(x_m^t)}_a \) from all \( M \) clients with a \textit{global non-linear ML} $f_s(.)$ to \textit{globally encode interdependencies}. The server minimizes its loss \( L_s \) and updates its global ML function $f_s$. It then communicates gradients \( \nabla_{{(x_m^t)}_a}L_s \) to the respective clients.
\end{enumerate}
The motivation behind the problem (interdependencies in OT clients), and the methodology of learning interdependencies in a federated setting is depicted in Fig. \ref{fig:flowchart}. 
\begin{figure}
    \centering
\includegraphics[width=\columnwidth]{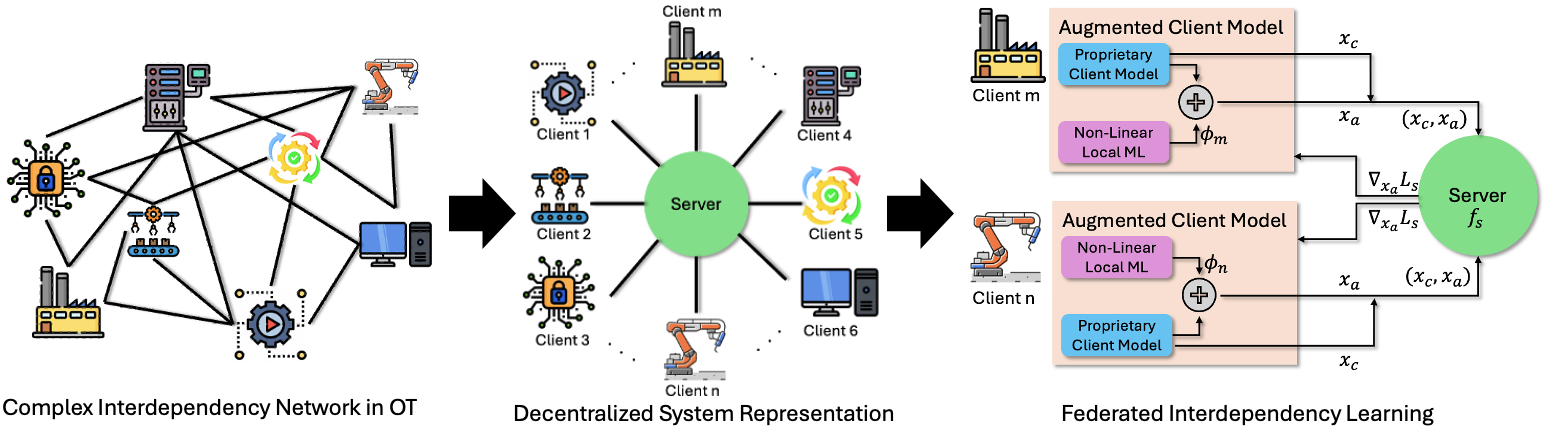} 
    \caption{The proposed federated framework for learning interdependencies across OT clients}
    \label{fig:flowchart}
\end{figure}

The trained framework captures the cross-client interdependencies as \textit{non-interpretable parameters} of ML functions (i.e., $\phi_m, f_s$). Thus, the framework can detect any local anomalies but lacks root cause analysis capabilities. To address this, we perform the following during \underline{inference}: 
\begin{enumerate}[leftmargin=*]
    \item \textbf{Client.} Detect anomalies locally using residual analysis-based binary anomaly flags \( Z_c \) and \( Z_a \) from the proprietary and augmented client models, respectively. Each client communicates their anomaly flag pair \((Z_c, Z_a)\) to the server.
    \item \textbf{Server.} The server then leverages the framework's understanding of cross-client interdependencies to distinguish between the clients that are associated with \textit{root cause} of an anomaly and those experiencing its \textit{propagated effect}. 
\end{enumerate}

%% file: Training.tex
\subsection{Proprietary Client Model}
\textbf{Model.} To handle non-linear dynamics, the proprietary client models are assumed to be \textit{Extended Kalman Filters} (EKF). Each client \( m \) utilizes local data \( y_m \in \mathbb{R}^{D_m}\) to compute states as follows:
\begin{equation}
        {(x^t_m)}_c = f_m\big({(\hat{x}^{t-1}_m)}_c\big), \hspace{0.5cm} \text{and} \hspace{0.5cm}
    {(\hat{x}^t_m)}_c = {(x^t_m)}_c + K_m^c\big(y^t_m - h_m({(x^t_m)}_c)\big) \label{clientmodel_update}
\end{equation}
where \({(\hat{x}^t_m)}_c \in \mathbb{R}^{P_m}\), and \({(x^t_m)}_c \in \mathbb{R}^{P_m}\) are the estimated and predicted states. \( f_m(\cdot) \) defines local state transitions, and \( h_m(\cdot) \) maps states to observations, and \( K_m^c \) is the Kalman gain computed as:
\begin{equation}
    K_m^c = P^-_m H_m^\top \big(H_m P^-_m H_m^\top + R_m\big)^{-1}, \hspace{0.5cm} \text{with} \hspace{0.5cm}     P^-_m = F_m P_m F_m^\top + Q_m
\end{equation}
where, \( H_m = \frac{\partial h_m}{\partial x}\big|_{(x^t_m)_c} \), and \( P^-_m \) and \( R_m \) represent the covariance of predicted state and measurement noise, respectively. Similarly, \( F_m = \frac{\partial f_m}{\partial x}\big|_{(\hat{x}^{t-1}_m)_c} \), \(P_m\) and \(Q_m\) are the covariance of the estimated state and process noise respectively. 

\textbf{Limitation.}
The local state transition function \( f_m(\cdot) \) models only client \( m \)'s dynamics using local data \( y^t_m \) and state \( {(\hat{x}^{t-1}_m)}_c \). Thus, it \textbf{fails to capture cross-client interdependencies} due to the lack of access to information (raw data or states) from other clients (i.e., $\forall$ clients $ i \in M \setminus\{m\}$).

\subsection{Augmented Client Model}
\textbf{State Augmentation.} To address this limitation, the proprietary client model is \textit{augmented} with a \textit{non-linear ML} function \( \phi_m(y^t_m; \theta_m) \), parameterized by \( \theta_m \), as follows:
\begin{equation}
    {(\hat{x}^t_m)}_a = {(\hat{x}^t_m)}_c + \phi_m(y^t_m; \theta_m), \hspace{0.5cm} \text{and} \hspace{0.5cm}
    {(x^t_m)}_a = f_m\big({(\hat{x}^{t-1}_m)}_a\big)
\end{equation}
\textbf{Loss Function and Param. Update.} The augmented client model minimizes a \textit{reconstruction loss} ${(L_m)}_a$ defined as:
$    {(L_m)}_a = \lVert y^t_m - h_m\big({(x^t_m)}_a\big) \rVert_2^2$. The client parameters \( \theta_m \) are then updated as:
\begin{align}
    \theta_m^{k+1} &= \theta_m^k - \eta_1 \nabla_{\theta_m^k} {(L_m)}_a - \eta_2 \nabla_{\theta_m^k} (L_s), \hspace{0.1cm} 
    \text{with} \hspace{0.2cm} \nabla_{\theta_m^k} (L_s) = \color{blue} \big(\nabla_{{(x^t_m)}_a} (L_s)\big) \cdot  \color{red} \big(\nabla_{\theta_m^k} {(x^t_m)}_a\big)
\end{align}
where \(\color{blue} \nabla_{{(\hat{x}^t_m)}_a} (L_s) \) is communicated from the server, and \(\color{red} \nabla_{\theta_m^k} {(\hat{x}^t_m)}_a \) is computed locally.

\textbf{Client to Server.} Each client \( m \) communicates the tuple of states \( [{(\hat{x}^{t-1}_m)}_c, {(x^t_m)}_a] \) to the server.

\subsection{Global Server Model}
The server model uses a \textit{non-linear ML} \( f_s\), serving as a proxy for a global state transition function. It predicts the future state of all $M$ clients using the current state estimates (received from the clients):  
\begin{equation}
    \big\{{(x^t_m)}_s\big\}_{m=1}^M = f_s\big(\{{(\hat{x}^{t-1}_m)}_c\}_{m=1}^M; \theta_s\big).
\end{equation}
\textbf{Loss Function and Param. Update.} The server minimizes the \textit{mean squared error} between it's predicted and client's augmented states: ${(L_s)} = \sum_{m=1}^M \big\| {(x^t_m)}_s - {(x^t_m)}_a \big\|_2^2$. It then updates \( \theta_s \) as: 
\begin{equation}
  \theta_s^{k+1} = \theta_s^k - \eta_3 \nabla_{\theta_s^k} (L_s)  
\end{equation}
\textbf{Server to Client:} The server communicates the gradient \(\color{blue}\nabla_{{(x^t_m)}_a} (L_s) \) to each client \( m \), enabling update of client model's parameter $\theta_m$. Algorithm 1 (see Appendix) provides the steps for training. 


%% file: Inference.tex
During inference \textbf{the server aims to identify the root cause of an anomaly without access to client states or raw data (decentralized)}. This is different from a complete causal discovery, which requires extensive conditional independence tests. Prior works \cite{ikramRootCauseAnalysis, lin2024root, shan2019diagnosis, Wang2023} similarly focus on root cause identification but in centralized settings.

\textbf{Metric.} An anomaly can be detected when a vector \( r \) has a significant deviation from its nominal behavior. If \( \mu \) and \( \Sigma \) represent the mean and covariance of \( r \) then we quantify this deviation as $d_M^2(.)$, which is further used to generate a binary anomaly flag Z as given in eq~\ref{eq:anomaly}. The threshold \( \tau \) is derived statistically from the vector $r$, \textit{\textbf{only during nominal operation}} of the client. 
\begin{equation}\label{eq:anomaly}
d_M^2(r) = (r - \mu)^\top \Sigma^{-1} (r - \mu), \quad \text{and} \quad 
Z =
\begin{cases}
1, & \text{if } d_M^2(r) > \tau \hspace{0.2cm} \textbf{(anomaly)}, \\
0, & \text{otherwise}
\end{cases}
\end{equation}
\textbf{Anomaly Detection.}  Each client \( m \) is equipped with two models \textbf{(1)} a proprietary client model, and \textbf{(2)} an augmented client model, generating residuals ${(r_m^t)}_c$, and ${(r_m^t)}_a$, respectively: 
\begin{equation}
    {(r_m^t)}_c = y^t_m - h_m({(x^t_m)}_c), \hspace{0.5cm} \text{and} \hspace{0.5cm}
    {(r_m^t)}_a = y^t_m - h_m({(x^t_m)}_a)
\end{equation}
These residuals are an integral part of the client models, appearing in eq \ref{clientmodel_update} and the \textit{loss function} of the augmented client model, respectively. They can be interpreted as the error in reconstructing the raw data $y_m^t$ using the two client models. For the remainder of this section, we drop indices $m$ and $t$ for notational brevity. Using $r_c$ and $r_a$, the client models can generate anomaly flags $Z_c$, and $Z_a$ respectively. With \( \tau_c \) and \( \tau_a \) as thresholds for the proprietary and augmented models respectively,
\begin{align}
Z_c =
\begin{cases}
1, & \text{if } d_M^2(r_c) > \tau_c \\
0, & \text{otherwise.}
\end{cases}, \hspace{0.5cm} \text{and} \hspace{0.5cm}
Z_a =
\begin{cases}
1, & \text{if } d_M^2(r_a) > \tau_a\\
0, & \text{otherwise.}
\end{cases}
\end{align}
\textbf{Communication (Inference).} Clients communicate the anomaly flags \( Z_c \) and \( Z_a \) to the server.

\textbf{Root Cause Analysis.} For any client, when \( Z_a = 1 \), the client  detects an anomaly. Since the augmented model is expected to learn interdependencies, so \( Z_a = 1 \) confirms a local fault (given interdependent clients and perfect training), \textbf{not} caused by interdependencies. Conversely, \( Z_c = 1 \) indicates an anomaly that \textbf{might be} interdependency-driven (propagation of effects). 

Given the above logic, we can say that the anomaly associated with a \textit{root cause } (RC) client originates from a local fault within that client, and identified when \( (Z_c, Z_a) = (1, 1) \). On the contrary, a \textit{propagated effect} (PE) occurs due to interdependency with the RC client, identified when \( (Z_c, Z_a) = (1, 0) \). For any two clients $C_1$, and $C_2$, Table~\ref{tab:primary_secondary_anomalies} provides a look-up table for RCA.

\begin{remark}
      A flag \( (Z_c, Z_a) = (1, 1) \) from an isolated client at a single time instant \textbf{cannot} confirm it as the RC, as it may be a false alarm, likely due to \textit{imperfect training} of the framework. PEs offer critical contextual evidence by revealing anomaly spread among interdependent clients. RCA requires monitoring $(Z_c, Z_a)$ across all clients over a long time horizon to identify consistent patterns.
\end{remark}

\begin{table*}
    \centering
    \caption{RCA for any two pair of clients $C_1$ and $C_2$ at any time $t$ using the anomaly flags}
    \begin{tabular}{|ccc|ccc|}
        \hline
         \multicolumn{1}{|c}{\textbf{Client \( C_1 \)}} & \multicolumn{1}{c}{\textbf{Client \( C_2 \)}} & \multicolumn{1}{c|}{\textbf{RCA}} & \multicolumn{1}{c}{\textbf{Client \( C_1 \)}} & \multicolumn{1}{c}{\textbf{Client \( C_2 \)}} & \multicolumn{1}{c|}{\textbf{RCA}} \\ 
         $(Z_c, Z_a)$ & $(Z_c, Z_a)$ & {(RC, PE)} & $(Z_c, Z_a)$ & $(Z_c, Z_a)$&  {(RC, PE)}\\ \hline
        (1, 1) & (0, 0) & (\( C_1 \), -- ) & (1, 0) & (1, 1) & (\( C_2 \), \( C_1 \)) \\ 
        (0, 0) & (1, 1) & ($C_2$, --) & (0, 0) & (0, 0) & No anomaly \\
        (1, 1) & (1, 0) & ($C_1$, $C_2$) & (1, 0) & (1, 0) & (Neither, Both)\\
         (0, 1) & (1, 1) & Imperfect training & (0, 1) & (0, 1) & Independent Clients \\
        (1, 1) & (0, 1) & Imperfect training & (1, 1) & (1, 1) & Violates Assumptions\\
        \hline
    \end{tabular}
\label{tab:primary_secondary_anomalies}
\end{table*}
\textbf{Assumptions.} The federated RCA methodology makes the following assumptions:\\
\textbullet{} \textit{\textbf{No Feedback Loops.}} The underlying causal structure of the system is a directed acyclic graph. \\
\textbullet{} \textit{\textbf{Causal Sufficiency.}} All clients share their anomaly flags and there are no latent confounders.\\
\textbullet{} \textit{\textbf{Multiple RCs.}} RC anomalies occurring at any client follow a Poisson arrival process.

\begin{proposition}\label{prop_poisson}
For \( M \) clients, let \( N_m(\Delta t) \) represent the number of root cause anomalies occurring at \( C_m \) within time \( \Delta t \). The probability of multiple clients experiencing root causes in \( \Delta t \) is negligible for small \( \Delta t \), i.e.,  
\(
P\left(\sum_{m=1}^M N_m(\Delta t) > 1\right) \to 0 \quad \text{as} \quad \Delta t \to 0.
\)
\end{proposition}
Proposition \ref{prop_poisson} ensures that multiple clients \textbf{cannot} share root causes anomalies simultaneously. Thus, there is always a unique root cause client that can be identified using Table \ref{tab:primary_secondary_anomalies}. The pseudocode for inference is provided in Algorithm 2 of the Appendix. 


%% file: Convergence.tex
This section demonstrated that the proposed framework's performance converges to a centralized oracle, confirming the framework's understanding of interdependencies. 

\begin{definition}[\textbf{Centralized Oracle}]
An \textit{Extended Kalman Filter} that utilizes data from all \( M \) clients, represented as \( y^t \), to predict the one-step ahead state \( x^t_o \) and estimate the current state \( \hat{x}^t_o \):
\begin{equation}
    x^t_o = f\big(\hat{x}^{t-1}_o\big) \hspace{0.5cm} \text{and,} \hspace{0.5cm}
    \hat{x}^t_o = x^t_o + K^o\big(y^t_o - h(x^t_o)\big)
\end{equation}
where $K^o$ is defined similar to the proprietary client model's $K^c_m$, with the difference that the Jacobians in the oracle are $F = \frac{\partial f}{\partial x} |_{\hat{x}^{t-1}_o}$, and $H = \frac{\partial h}{\partial x} |_{\hat{x}^{t-1}_o}$. We utilize an operator \( {Extract}_m \) to isolate the components of a vector (e.g., \( x^t_o, \hat{x}^t_o, y^t_o \)) or the sub-matrix of a matrix (e.g., \( K^o \)) corresponding only to client \( m \). 
 
\end{definition}
\begin{assumption}
The client loss function \({(L_m)}_a)\) is convex and \(\mathcal{L}_{L_m}\)-Lipschitz smooth  w.r.t. \(\theta_m\).
\end{assumption}

\begin{assumption}
The non-linear local ML function \(\phi_m(y^t_m; \theta_m)\) is \(\mathcal{L}_{\theta_{m}} \)-Lipschitz w.r.t. \(\theta_m\)
\end{assumption}

\begin{definition}[\textbf{Irreducible Local Error} \(\epsilon^*\)]
\label{def:irreducible_error}
Consider the optimal parameter \(\theta_m^*\) for the local ML function \(\phi_m(\cdot; \theta)\).  The \emph{\textbf{irreducible local error}}  \(\epsilon^* \;:=\;
\mathbb{E}\bigl[\bigl\|\phi_m\bigl(y_m^{t-1}; \theta_m^*\bigr) - \sum_{n \neq m} \bigl(\hat{x}_n^{t-1}\bigr)_o\bigr\|\bigr]\) is the discrepancy between the optimal local ML prediction and the sum of other clients’ states. 

If trained perfectly, \(\epsilon^*\) may be negligible or zero; otherwise, it quantifies the \textbf{training error of the local ML model in capturing interdependencies} across clients.
\end{definition}
\begin{lemma}[\textbf{Local Learning of Interdependency }]\label{lem:phi_approx_error}
For fixed step size $\eta_1, \eta_2 \leq 1/L$, we obtain, $\mathbb{E}\bigl[\|\phi_m(y_m^{t-1}; \theta_m^k) - \sum_{n \neq m} (\hat{x}_n^{t-1})_o\|\bigr] \leq \mathcal{L}_{\theta_{m}} \|\theta_m^k - \theta_m^*\| + \epsilon^*$. 
As \(k \to \infty\), \(\|\theta_m^k - \theta_m^*\| \to 0\), and the limit of the above norm difference is bounded by \(\epsilon^*\).
\end{lemma}
Lemma \ref{lem:phi_approx_error} shows that the \textbf{local ML function at each client approximates the contribution from other clients}. Next we assume bounds on the terms involved in the EKF-based proprietary client model to prove the convergence of our decentralized framework to the centralized oracle.
\begin{assumption}
\label{ass:bounded_state_space}
The states of each client $m$ is bounded, i.e., $\mathbb{E}\bigl[\|{(x^t_m)}_o\|\bigr] \leq C_x \forall t$, where $C_x > 0$
\end{assumption}
\begin{assumption}
\label{ass:lipschitz_smooth}
The (local) state transition function \(f_m(\cdot)\) and the measurement function \(h_m(\cdot)\) are Lipschitz continuous with Lipschitz constants \(\mathcal{L}_{f_m}\) and \(\mathcal{L}_{h_m}\), respectively.
\end{assumption}
The state transition functions of the proprietary client model and the centralized oracle \textbf{differ structurally}. Each client \( m \) uses \( f_m \) in its local client model, while the centralized oracle uses \( f \) and performs \( {Extract}_m(f(\mathbf{z})) \). Proposition \ref{prop:bounded_structuraldifference} shows that the structural difference is bounded. 
\begin{proposition}[\textbf{Structural Diff.}]\label{prop:bounded_structuraldifference}
For client $m$, $\exists \delta_m$ s.t., $\mathbb{E}\bigl[\|f_m(z_m) - \text{Extract}_m(f(\mathbf{z}))\|\bigr] \leq \delta_m$.
\end{proposition}
\begin{definition}[\textbf{Residual of Centralized Oracle}]
    For any client $m$, the residual of the EKF-based centralized oracle is given by ${(r_m^t)}_o := y_m^t - h_m({(x_m^t)}_o)$. 
\end{definition}
\begin{theorem}[\textbf{Local Convergence to Oracle}]
\label{thm:predicted_state_convergence}
Let \(\Delta K_m := K_m^c - \text{Extract}_m(K^o)\). Then, for fixed step sizes \(\eta_1, \eta_2 \leq 1/L\), if $1 + \|\Delta K_m\| \mathcal{L}_{h_m} < \frac{1}{\mathcal{L}_{f_m}}$ then the predicted state of the augmented client model \((x^{t,k}_m)_a\) converges to that of a centralized oracle \((x^t_m)_o\) in steady state:
$  \lim_{k \to \infty} \mathbb{E}\bigl[\|(x^{t,k}_m)_a - (x^t_m)_o\|\bigr] 
\;=\mathcal{O}\bigl(\mathcal{L}_{f_m}(\rho + \epsilon^*) + \delta_m\bigr)
\hspace{0.1cm} \text{where}, \hspace{0.1cm} \rho := \frac{\|\Delta K_m\| \sup_t \|{(r_m^t)}_o\| + \mathcal{L}_{h_m} C_x + \mathcal{L}_{f_m} (M-1) C_x}{1 - (1 + \|\Delta K_m\| \mathcal{L}_{h_m}) \mathcal{L}_{f_m}} $
\end{theorem} 

Theorem~\ref{thm:predicted_state_convergence} shows that the predicted state of the augmented client model converges to that of a centralized oracle in the steady state. For \textbf{perfect learning of interdependency (\(\epsilon^* \to 0\)) and small structural difference (\(\delta_m \to 0)\)}, the augmented client model converges to the centralized oracle up to a constant $\mathcal{L}_{f_m} \rho$. This constant is a function of the parameters of the state space models. 
\begin{definition}[\textbf{Irreducible Global Error}]
\label{def:irreducible_global_error}
The \emph{\textbf{irreducible global error}} \(\sigma^*\) is defined as
$\sigma^* := \mathbb{E}\Bigl[\Bigl\| f_s\big(\{{(\hat{x}^{t}_m)}_c\}_{m=1}^M; \theta_s^*\big) - f\big(\hat{x}^{t}_o\big) \Bigr\|\Bigr]$
where \(\theta_s^*\) represents the optimal param. of the server model. 

This error captures the \textbf{training error of the global ML model in capturing interdependencies} across all $M$ clients. When trained perfectly, the server converges to the centralized oracle or, \(\sigma^* \to 0\).
\end{definition}
\begin{assumption}
 $f_s$ of the global server model is $\mathcal{L}_{f_s}$-Lipschitz continuous w.r.t to its input. 
\end{assumption}
\begin{theorem}[\textbf{Global Server Convergence to Oracle}]
\label{thm:global_server_convergence}
Let \(\Delta K_m := K_m^c - \text{Extract}_m(K^o)\). Then, for fixed learning rates \(\eta_3 \leq 1/L\), if \(1 + \|\Delta K_m\| \mathcal{L}_{h_m} < \frac{1}{\mathcal{L}_{f_m}}\), the predicted state of the global server model \((x_s^{t,k})\) converges to that of a centralized oracle \((x_o^t)\) in steady state:
$
\lim_{k \to \infty} \mathbb{E}\bigl[\|x_s^{t,k} - x_o^t\|\bigr] 
\;=\mathcal{O}\bigl(\mathcal{L}_{f_s} M \rho + \sigma^*\bigr)
\hspace{0.1cm} \text{where}, \hspace{0.1cm}
\rho := \frac{\|\Delta K_m\| \sup_t \|{(r_m^t)}_o\| + \mathcal{L}_{h_m} C_x + \mathcal{L}_{f_m} (M-1) C_x}{1 - (1 + \|\Delta K_m\| \mathcal{L}_{h_m}) \mathcal{L}_{f_m}} $
\end{theorem}
The result presented in Theorem \ref{thm:global_server_convergence} implies that the \textbf{understanding of interdependencies by the server model converges to that of a centralized oracle} up to a constant depending on the parameters of the state space models, and the irreducible error of the global server model. 

%% file: PrivacyAnalysis.tex
In this section we analyze how communications during training and inference are protected to prevent revealing sensitive client data. During training, the communicated local and augmented client models' states can be perturbed as \( {(\tilde{x}_{m}^{t-1})}_c = {(\hat{x}_{m}^{t-1})}_c + \mathcal{N}(0, \sigma_c^2) \) and \( {(\tilde{x}_{m}^t)}_a = {(x_{m}^t)}_a + \mathcal{N}(0, \sigma_a^2) \).
\begin{theorem}[\textbf{Client-to-Server Communication}]\label{thm:client_to_server}
At each time step \( t \), the mechanisms by which client \( m \) sends \({(\tilde{x}_{m}^{t-1})}_c\) (state communication) and \({(\tilde{x}_{m}^t)}_a\) (augmented communication) to the server satisfy \((\varepsilon, \delta)\)-differential privacy w.r.t. the client’s local data \( y_m^t \). This is achieved by adding Gaussian noise with standard deviations:
\(
    \sigma_c \geq \frac{2 \|\Delta K_m\| (\mathcal{L}_{h_m} C_x + \sup_t \|r_m^t\|) \sqrt{2 \ln(1.25 / \delta_c)}}{\varepsilon_c}, 
    \sigma_a \geq \frac{2 (\mathcal{L}_{f_m} C_x + \mathcal{L}_{\theta_m} \sup_t \|y_m^t\|) \sqrt{2 \ln(1.25 / \delta_a)}}{\varepsilon_a}.
\)
The total privacy budget is \( \varepsilon = \varepsilon_c + \varepsilon_a \) and \( \delta = \delta_c + \delta_a \).
\end{theorem}
Theorem \ref{thm:client_to_server} says that the proprietary and augmented client states sent to the server can be perturbed with noise proportional to their sensitivities, preventing leakage of client data. When the server communicates back to the clients, the gradients \( g_m^t := 
 \color{blue} \big(\nabla_{{(\hat{x}^t_m)}_a} L_s\big)\) can be perturbed as:
\(
\tilde{g}_m^t = \text{clip}(g_m^t, C_g) + \mathcal{N}(0, \sigma_g^2),
\)
where \( \text{clip}(g_m^t, C_g) \) scales \( \lVert g_m^t \rVert_2 \) to \( C_g \) if \(\|g_m^t\|_2 > C_g\).
\begin{theorem}[\textbf{Server-to-Client Communication}]\label{thm:server_to_client}
At each time step \( t \), the mechanism by which the server sends \(\tilde{g}_m^t\) (gradient updates) to client \( m \) satisfies \((\varepsilon,  \delta)\)-differential privacy w.r.t. any single client’s data. The sensitivity of the clipped gradient is bounded by \( C_g := \max_{t} \|g_m^t\|_2 \), and Gaussian noise with standard deviation:
\(
\sigma_g \geq \frac{2 C_g \sqrt{2 \ln(1.25 / \delta)}}{\varepsilon},
\)
is added to ensure privacy. 
\end{theorem}

Theorem \ref{thm:server_to_client} guarantees that gradients sent by the server can be perturbed, protecting data of any single client. Similar to training, during inference, the binary anomaly flags can be perturbed with Bernoulli noise, resulting in Theorem \ref{thm:privacy_inferencing}. 
\begin{theorem}[\textbf{Privacy during Inferencing}]\label{thm:privacy_inferencing}
During inference, the mechanism by which each client communicates \(\tilde{Z}_c^t\), and \(\tilde{Z}_a^t\) satisfies \(\varepsilon\)-differential privacy w.r.t. the client’s residuals. This is achieved through randomized response with probabilities \( p_c = \frac{e^{\varepsilon_c}}{1 + e^{\varepsilon_c}} \) and \( p_a = \frac{e^{\varepsilon_a}}{1 + e^{\varepsilon_a}} \).
\end{theorem}

%% file: Experiments.tex
\subsection{Synthetic Datasets}
 \textbf{Data Generation.} We simulate an \( M \)-client non-linear state-space model where the root cause client evolves as: \( x_m^t = f_m(x_m^{t-1}) + \epsilon_m \), and clients exhibiting propagated effects follow: \( x_n^t = f_n(x_n^{t-1}) + g_{dep.}^n(x_m^{t-1}) + \epsilon_n \), with \( \epsilon_m \sim \mathcal{N}(0, Q_m) \hspace{0.1cm} \forall m \). Directed interdependencies are captured in \( g_{dep.}^n(x_m^{t-1}) \). Raw data are mapped via \( y_m^t = h(x_m^t) + \zeta_m \), where \( \zeta_m \sim \mathcal{N}(0, R_m) \hspace{0.1cm} \forall m \).  For all \( M \) clients, \( f_m(.) \) and \( g_{dep.}^m(.) \) are modeled as LSTMs, 
 while \( h_m(.) \) are implemented as fully connected networks. The ``\textbf{\textit{Architectural details}}'' of $f_m, h_m, g_{dep}$ are in the Appendix. Except for “\textit{Scalability studies}”, experiments use \( M = 2 \), \( D_m = 4 \), and \( P_m = 2 \hspace{0.1cm} \forall m \in \{1,2\} \) to ensure a controlled and interpretable evaluation. The true network structure follows \( C_1 \to C_2 \), with \( C_1 \) influencing \( C_2 \). 
 Due to space limitations, further details on ``\textbf{\textit{Experimental Settings}}'' are provided in the Appendix. 

\textbf{Learning Interdependencies.} After training, we assess the framework's ability to learn cross-client interdependencies. Fig. \ref{fig:synthetic_correlation} shows state correlations in synthetic datasets, where augmented client models and the server model closely match the centralized oracle, differing significantly from proprietary client models. This highlights their effectiveness in capturing cross-client interdependencies. 

\begin{figure}[H]
    \centering
  \subfloat{%
       \includegraphics[width=0.24\columnwidth]{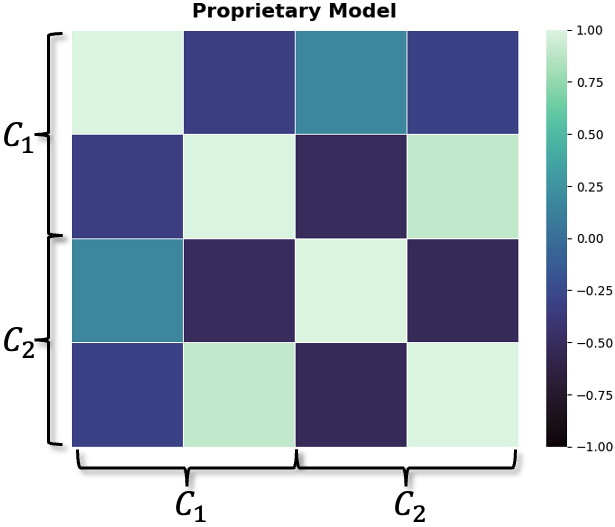}}
      \subfloat{%
        \includegraphics[width=0.24\columnwidth]{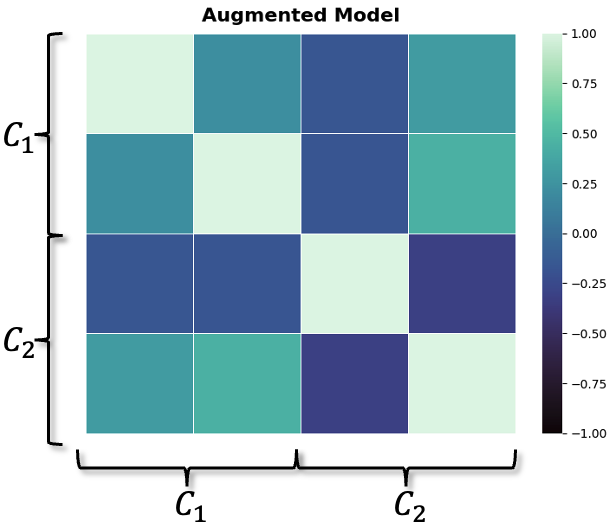}}
  \subfloat{%
       \includegraphics[width=0.24\columnwidth]{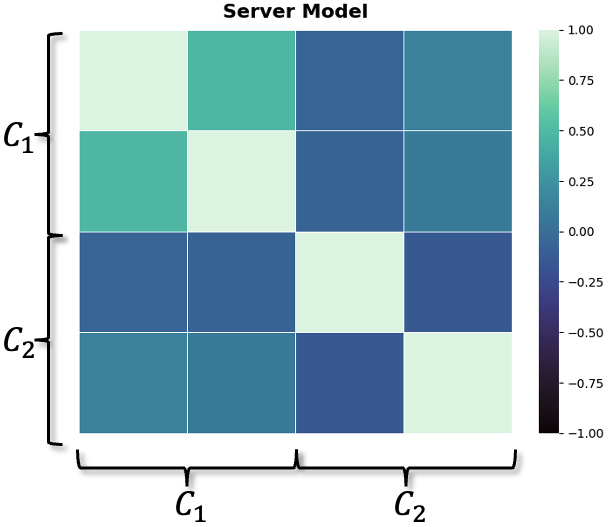}} 
       \subfloat{%
       \includegraphics[width=0.24\columnwidth]{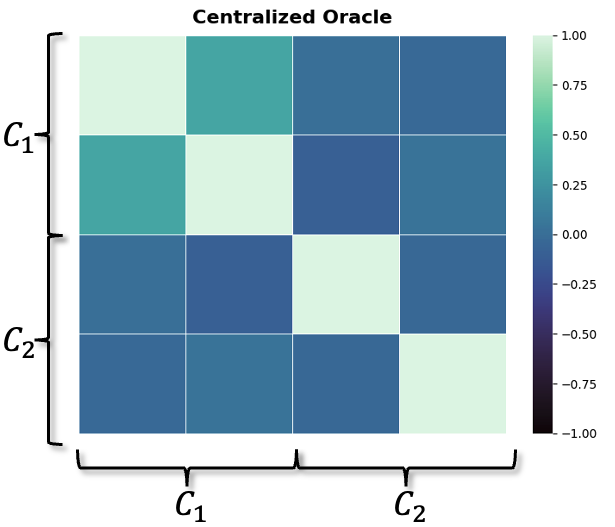}}
  \caption{Correlation in the states of clients $C_1$, and $C_2$ in their proprietary model (left), augmented model (second to the left), server model (second to the right), and centralized oracle (right)}
  \label{fig:synthetic_correlation} 
\end{figure}
\textbf{Baselines.} We compare against the following baselines: (1) \textit{Only Proprietary Model:} EKF-based proprietary client models without local ML augmentation or a global server model, evaluating the limitations relying solely on local data. (2) \textit{Without Proprietary Model:} Clients use end-to-end ML instead of EKF-based proprietary models, testing the performance without a domain-specific proprietary model. (3) \textit{Pre-trained Clients \cite{Ma2023}:} Clients pre-train augmented models independently and share state representations for training a global server model, assessing the impact of federated interdependency learning without continuous server feedback. (4) \textit{Centralized Oracle:} An EKF with access to all client states and raw data, serving as an upper bound with known interdependencies.

\textbf{Evaluation Metric.} Anomaly detection is evaluated using \textbf{(1)} \textit{\(ARL_0\)}, which measures the expected time passed before a false alarm under nominal conditions, and \textbf{(2)} \textit{\(ARL_1\)}, which quantifies the average time to detect a true anomaly. RCA performance is assessed with \textbf{(1)} \textit{Precision}, the proportion of correctly identified RCs, and \textbf{(2)} \textit{Recall}, the proportion of true RCs successfully identified.

\textbf{Anomaly Detection.}  The inference step for all four baselines follows the procedure in Section \ref{section:inferencing}, utilizing anomaly flags for RCA. Thresholds \(\tau_c\) and \(\tau_a\) are computed statistically using nominal data. The anomaly detection for synthetic datasets is given in Table \ref{tab:synthetic:combined_performance}. The \textit{Only Proprietary Model}, which does not use ML or encode any interdependencies, achieves the highest $ARL_0$ (fewest false alarms) but suffers from the highest $ARL_1$ (slowest detection). Despite being decentralized, our approach achieves the close enough performance as the \textit{Centralized Oracle}, which has direct access to interdependencies $g_{dep.}$. Furthermore, our approach provides same $ARL$ values as \textit{Pre-trained Clients} and \textit{Without Proprietary Model} baselines, suggesting the dominance of ML in detection. 
\begin{table}[H]
\centering
\caption{Performance of anomaly detection and root cause analysis approach discussed in Section \ref{section:inferencing}}
\begin{tabular}{|c|cc|ccc|}
\hline
 & \multicolumn{2}{c|}{\textbf{Anomaly Detection}} & \multicolumn{3}{c|}{\textbf{Root Cause Analysis (RCA)}} \\ \cline{2-6}
    \textbf{Model Used}               & \textbf{$ARL_0$} & \textbf{$ARL_1$} & \textbf{Precision} & \textbf{Recall} & \textbf{$F_1$ score} \\ \hline
Only Proprietary Model     & 36.88 & 157.8  & \multicolumn{3}{c|}{\textit{Cannot perform RCA as $Z_a$ is absent}} \\
Without Proprietary Model  & 27.82 & 112.38 & \multicolumn{3}{c|}{\textit{Cannot perform RCA as $Z_c$ is absent}} \\
Pre-trained Clients        & 27.82 & 112.38 & 0.70 & 0.49 & 0.576 \\
\textbf{Our Framework}      & 27.82 & 112.38 & 0.73 & 0.57 & 0.640 \\
Centralized Oracle         & 22.9  & 112.75 & \multicolumn{3}{c|}{\textit{Trivial with known interdependencies}} \\ \hline
\end{tabular}
\label{tab:synthetic:combined_performance}
\end{table}

\textbf{Root Cause Analysis.} Once the anomalies are detected, we perform RCA on the entire system. Table \ref{tab:synthetic:combined_performance} provides the performance for our approach and \textit{Pre-trained Clients}. Since \textit{Only Proprietary Client Model}, and \textit{Without Proprietary Client Model} \textbf{do not have $Z_a$, and $Z_c$} respectively, we \textbf{cannot utilize our RCA approach for them}. Furthermore, \textit{Centralized Oracle} has access to interdependencies, making it trivial for RCA. Our approach achieves high precision, without direct access to interdependencies. However, its low recall suggests potential gaps in identifying all root causes, warranting further refinement in our federated RCA approach. Despite same anomaly detection performance, our method marginally outperforms \textit{Pre-trained Clients}, emphasizing the advantages of server-to-client communication in RCA.

\begin{figure}[H]
    \centering
  \subfloat{%
       \includegraphics[width=0.325\columnwidth]{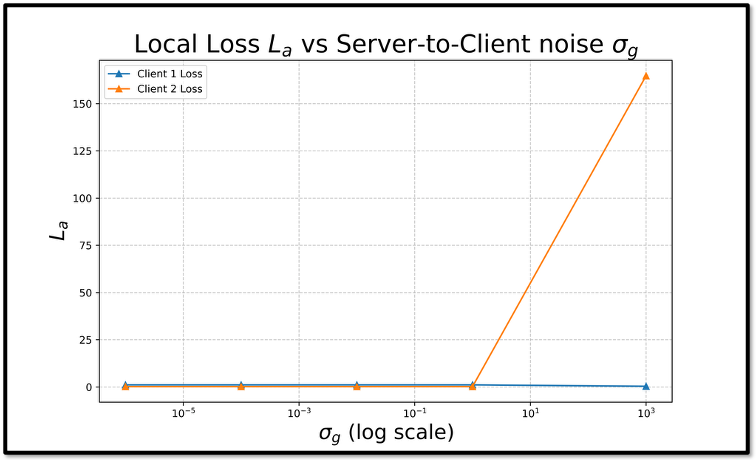}}
      \subfloat{%
        \includegraphics[width=0.325\columnwidth]{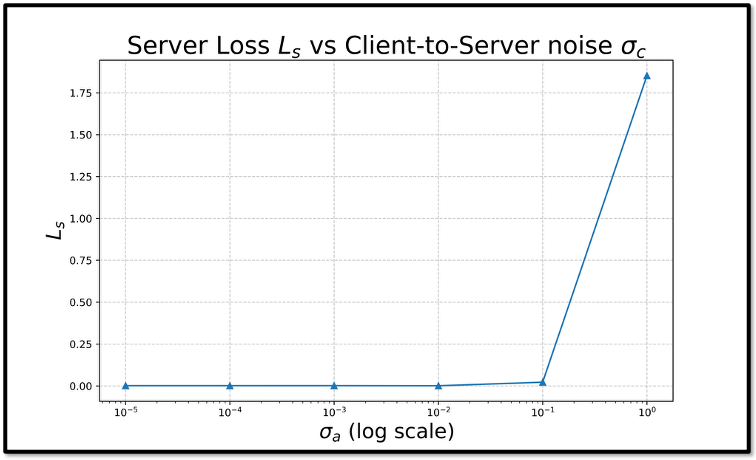}}
  \subfloat{%
       \includegraphics[width=0.325\columnwidth]{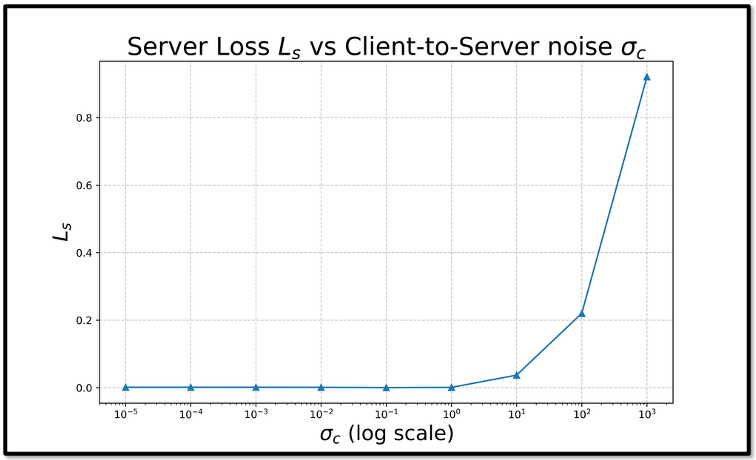}} 
  \caption{Loss under Gaussian noise in (a) server-to-client and (b-c) client-to-server communications.}
  \label{fig:privacy} 
\end{figure}
\textbf{Privacy Analysis.} We add Gaussian noise with varying standard deviation to communications: \textbf{(i)} server-to-client and \textbf{(ii)} client-to-server (Fig.~\ref{fig:privacy}). We observe that server and client losses remain stable until impractically high noise levels ($10^1$ and $10^5$ respectively), demonstrating training robustness under privacy constraints. We also investigate performance evaluation of \textbf{(i)} anomaly detection and \textbf{(ii)} root cause analysis (Fig.~\ref{fig:performance}) under randomized anomaly flags. Metrics remain stable with varying $p_c$ (i.e., noise added to $Z_c$) but show sensitivity to changes in $p_a$ (i.e., noise in $Z_a$).
\begin{figure}
    \centering
  \subfloat{%
       \includegraphics[width=0.24\textwidth]{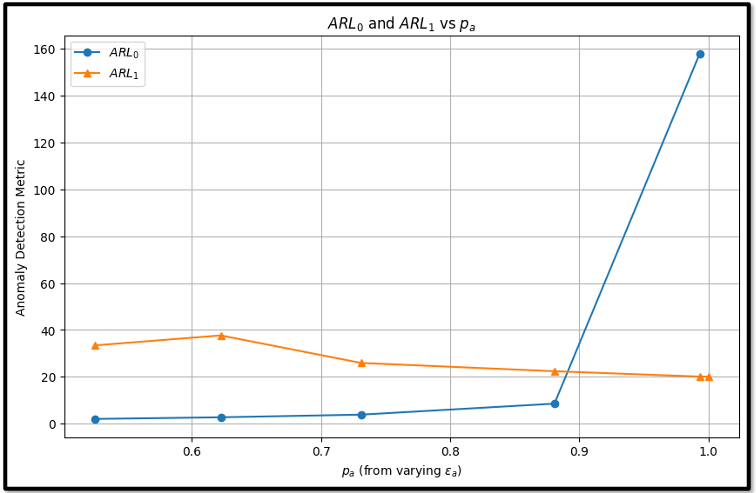}}
      \subfloat{%
        \includegraphics[width=0.24\textwidth]{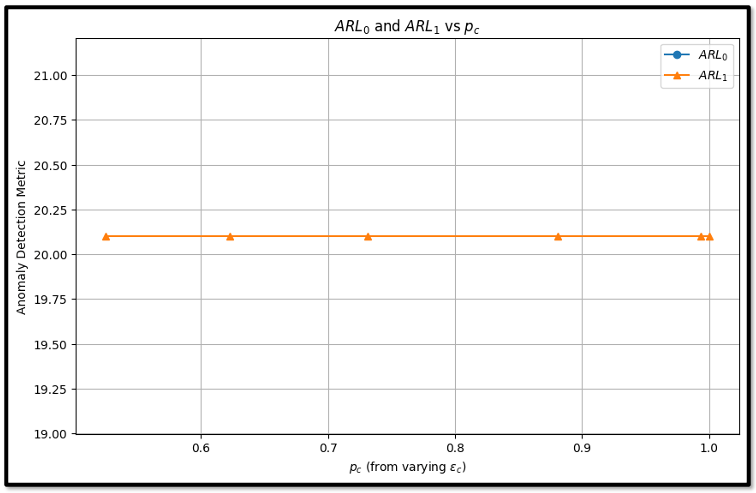}}
        \subfloat{%
       \includegraphics[width=0.236\textwidth]{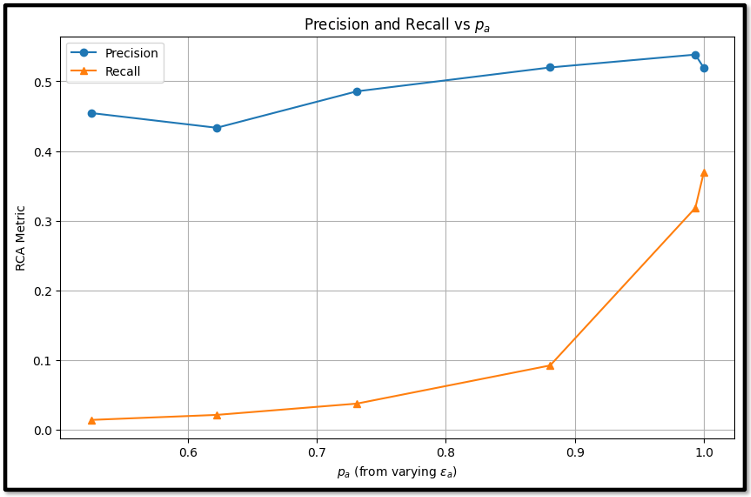}}
      \subfloat{%
        \includegraphics[width=0.24\textwidth]{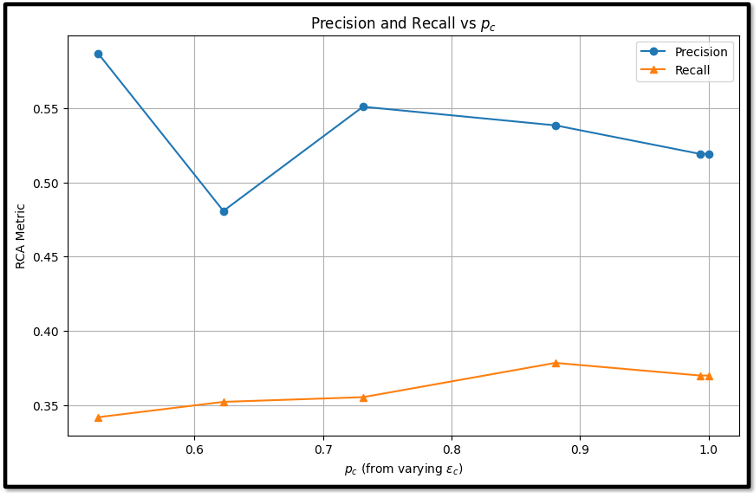}}
  \caption{Anomaly detection performance showing (a) sensitivity to $p_a$ and (b) robustness to $p_c$. RCA performance under randomized anomaly flags showing sensitivity to (c) $p_a$ and (d) $p_c$}
  \label{fig:performance} 
\end{figure}

\textbf{Scalability Studies.}
We assess the framework’s scalability by analyzing the final training loss as a function of communication cost per round. We vary the number of clients (\(M \in \{2, 4, 8, 16\}\)) keeping the state and data dimensions fixed. Table \ref{tab:synthetic_scalability} gives the server loss ($L_s$) vs. comm. overhead. 
\begin{table}[H]
\centering
\caption{Scalability Studies}
\renewcommand{\arraystretch}{1.2}
\begin{tabular}{@{}|cc|cc|@{}}
\hline
\textbf{Comm. Overhead} & \textbf{Server Loss} & \textbf{Comm. Overhead} & \textbf{Server Loss} \\
\textbf{(bytes/round)} & \textbf{($L_s$)} & \textbf{(bytes/round)} & \textbf{($L_s$)} \\ \hline
32  & 0.0305 & 128 & 0.0263 \\
64  & 0.0129 & 256 & 0.0412 \\
\hline
\end{tabular}
\label{tab:synthetic_scalability}
\end{table}

\textbf{Threshold Sensitivity.} All previous experiments were conducted with $\tau_a$ set at the 95th percentile of Mahalnobis distance. The framework's sensitivity to different $\tau_a$
is studied in the Appendix.

\subsection{Real-world Industrial Dataset}

\textbf{Description.} We evaluate our framework on a real-world industrial cybersecurity dataset: \textit{Hardware-In-the-Loop Augmented Industrial Control System (HAI)} \cite{github_HAI}. HAI emulates steam-turbine and pumped-storage hydropower generation, comprising four processes— \textit{Boiler}, \textit{Turbine}, \textit{Water Treatment}, and \textit{HIL Simulation}, each treated as a separate client. The dataset includes cyberattacks that introduce anomalies in a single process (client), which propagate due to inter-process dependencies.

We specify the ``\textit{\textbf{Experimental Settings}}'', and ``\textbf{\textit{Pre-processing Steps}}'' in the Appendix. 

\textbf{Training.} We train the dataset on the testbed's nominal operation (i.e., no cyberattacks) to learn cross-client interdependencies. Results (provided in the Appendix) depict that the framework learned a correlation between the \textit{Boiler}, \textit{Turbine}, \textit{Water Treatment} processes without moving the sensor data. 

\textbf{Inference.} The test dataset includes cyberattack-based anomalies in the testbed's operation. The ground-truth label of the anomaly, as well as documentation on the root cause, are also provided for validation. We use our methodology to detect these anomalies and identify the root cause process. Due to space limitations, we provide the results in the Appendix. 

%% file: ConclusionANDLimitations.tex

This work proposes a federated cross-client interdependency learning framework that enables RCA in decentralized industrial systems with nonlinear and heterogeneous IoT data. The approach augments unmodifiable proprietary client models with local ML components and coordinates interdependency learning through a central server. Anomalies are detected locally, and only binary flags are shared with the server for a decentralized RCA. The method is evaluated on synthetic and real-world industrial dataset, validating the claims made about interdependency learning and RCA. 

Our method has certain key limitations: RCA is performed only at the server, requiring clients to share anomaly flags to confirm root causes. The convergence analysis assumes proprietary client models follow an EKF structure, limiting generalizability to other model types. Empirical performance is sensitive to threshold, data noise, and the architecture of EKF, ML, and interdependency functions.

%% file: Appendix/Appendix.tex
\section{Additional Results}
\subsection{Details on Synthetic Datasets}
\subfile{Synthetic.tex}

\subsection{Real World Industrial Dataset}
\subfile{RealWorld.tex}

\section{Proofs}
\subfile{proofs.tex}
\newpage
\section{Pseudocode}
\subfile{pseudocode.tex}

%% file: Appendix/Synthetic.tex
\subsubsection{Experimental Settings}
We generate synthetic data using LSTM-based state transition models, dependency functions, and measurement functions. The training (and nominal phase of inference) data generation process ensures smooth variations without abrupt anomalies, maintaining consistency with real-world settings. 

The system dynamics are modeled using an LSTM-based state transition function $f_m$, which maps the previous state to the next state. Each client maintains an independent non-linear transition model with a hidden dimension of 16:
\begin{equation}
    x_{t+1} = f_m(x_t),
\end{equation}
where $f_m$ is a single-layer LSTM followed by a fully connected layer. Interdependencies among clients are captured using an LSTM-based function $g_{\text{dep}}$, which models weakened dependencies for nominal data to minimize anomaly propagation. It consists of a two-layer LSTM with 64 hidden units and a fully connected layer, with a scaling factor $\alpha = 0.5$ to reduce dependency strength:
\begin{equation}
    g_{\text{dep}}(x_t) = \alpha \cdot \tanh(W_{\text{dep}} \cdot \text{LSTM}(x_t)),
\end{equation}
where $W_{\text{dep}}$ is a learned transformation matrix. Observations are generated using a measurement function $h_m$, which maps the latent states to measurements. It consists of a two-layer feedforward network with a SELU activation in the first layer and an identity function in the second:
\begin{equation}
    y_t = h_m(x_t) + \eta_t, \quad \eta_t \sim \mathcal{N}(0, \sigma_m^2).
\end{equation}
A learnable scaling factor of 50 is applied to ensure comparability with anomalous data. To ensure smooth variations in nominal states, a sinusoidal target vector is generated:
\begin{equation}
    x_t = A_t \cdot \sin\left(\frac{2\pi t}{T_s}\right) + \epsilon_t, \quad \epsilon_t \sim \mathcal{N}(0, \sigma_s^2),
\end{equation}
where $A_t$ represents a slowly varying step function and $T_s=20$ is the step size. The system consists of $M$ clients, each with $P$ state dimensions and $D$ measurement dimensions. At each timestep $t$, the root cause client updates its state using $f_m$. Other clients update based on both $f_m$ and dependencies from $g_{\text{dep}}$. Process noise $\mathcal{N}(0, \sigma_p^2)$ is then added to simulate real-world variations. Finally Measurements are obtained through $h_m$, with additional measurement noise $\mathcal{N}(0, \sigma_m^2)$.

We introduce sustained anomalies into the synthetic dataset to evaluate the framework's ability to detect and localize anomalies. The anomalous data generation follows a structured approach to ensure consistency with real-world industrial fault patterns while maintaining the statistical properties of nominal data. Anomalies are injected into the target vector using controlled perturbations:
\begin{equation}
    x_t = A_t \cdot \sin\left(\frac{2\pi t}{T_s}\right) + \epsilon_t + \delta_t,
\end{equation}
where, $A_t$ represents smooth nominal variations, $\epsilon_t \sim \mathcal{N}(0, \sigma_s^2)$ is small stochastic noise, and $\delta_t$ is the anomaly term, which introduces sustained deviations.
Anomalies are injected periodically every five intervals, lasting for a sustained duration of $10$ timesteps. This simulates real-world scenarios where faults persist over time rather than appearing as transient spikes.

To create anomalies, a magnitude shift of $\Delta = 2.0$ is applied across all affected clients:
\begin{equation}
    \delta_t =
    \begin{cases}
        \Delta, & \text{if } t \in \mathcal{T}_{\text{anomaly}} \\
        0, & \text{otherwise}
    \end{cases}
\end{equation}
where $\mathcal{T}_{\text{anomaly}}$ represents the set of timesteps where anomalies occur. Anomalies are applied uniformly across all clients, ensuring a system-wide deviation. Each anomaly event is logged to track when and where anomalies were introduced:
\begin{equation}
    \mathcal{A} = \{(t, m) \mid t \in \mathcal{T}_{\text{anomaly}}, m \in M\},
\end{equation}
where $\mathcal{A}$ stores the affected timesteps and clients. 

\subsubsection{Threshold Sensitivity}\label{sec:ThresholdSensitivity}
In Fig. \ref{fig:synthetic:sensitivity_tau_a}, we evaluate the impact of the anomaly detection threshold $\tau_a$ by varying its percentile from $75$ to $97.5$ and measuring its effect on ARL$_0$ (false alarms), ARL$_1$ (timely detections), and RCA (precision and recall). The results indicate that while precision and recall remain relatively stable across different $\tau_a$ values, anomaly detection performance is highly sensitive to $\tau_a$. In particular, $ARL_1$ (delay in detection) exhibits an exponential increase with a higher threshold $\tau_a$. This emphasizes the challenge of selecting an optimal $\tau_a$ and mandates further investigation. 
\begin{figure}[H]
    \centering
  \subfloat{%
       \includegraphics[width=0.48\columnwidth]{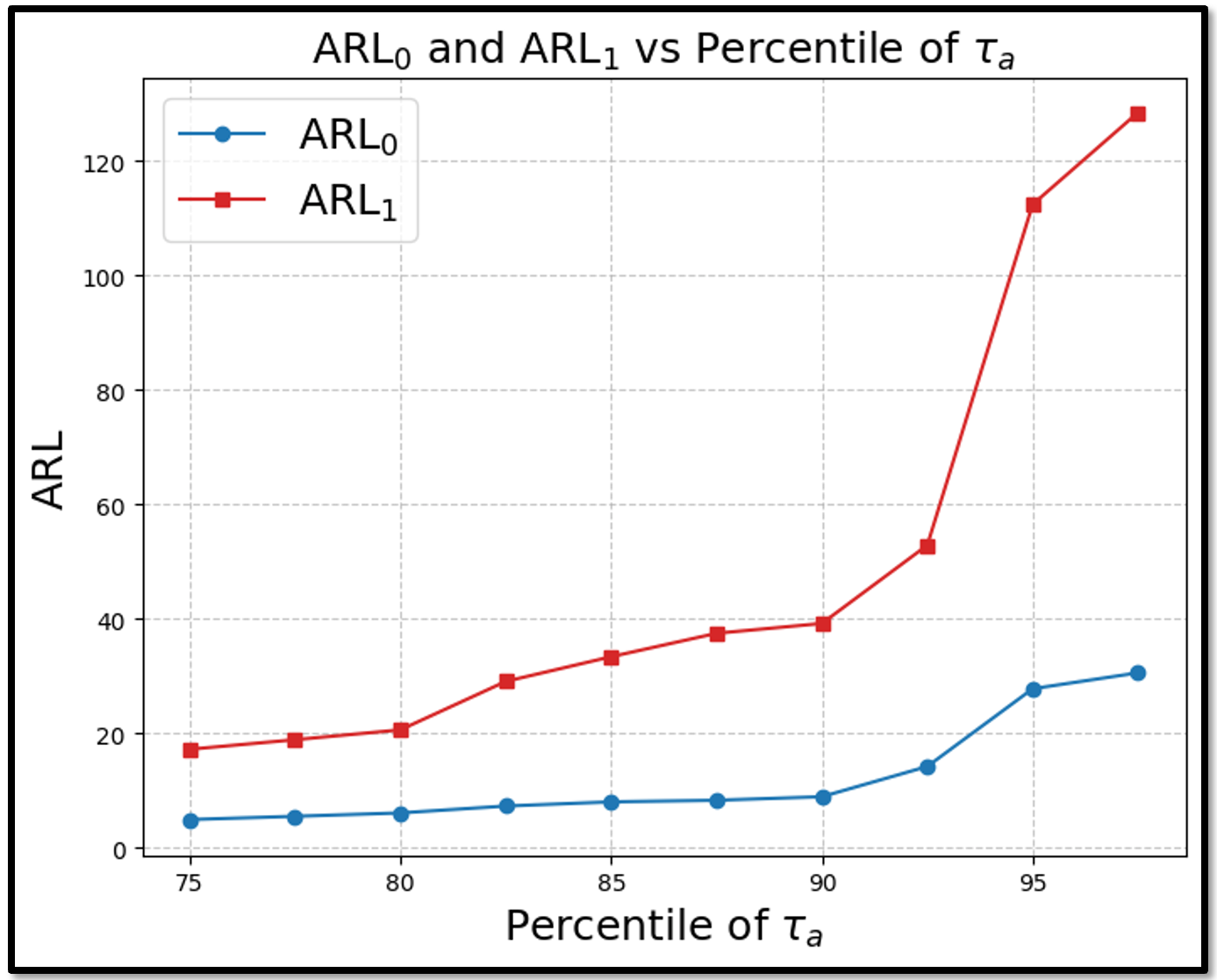}}
      \subfloat{%
        \includegraphics[width=0.475\columnwidth]{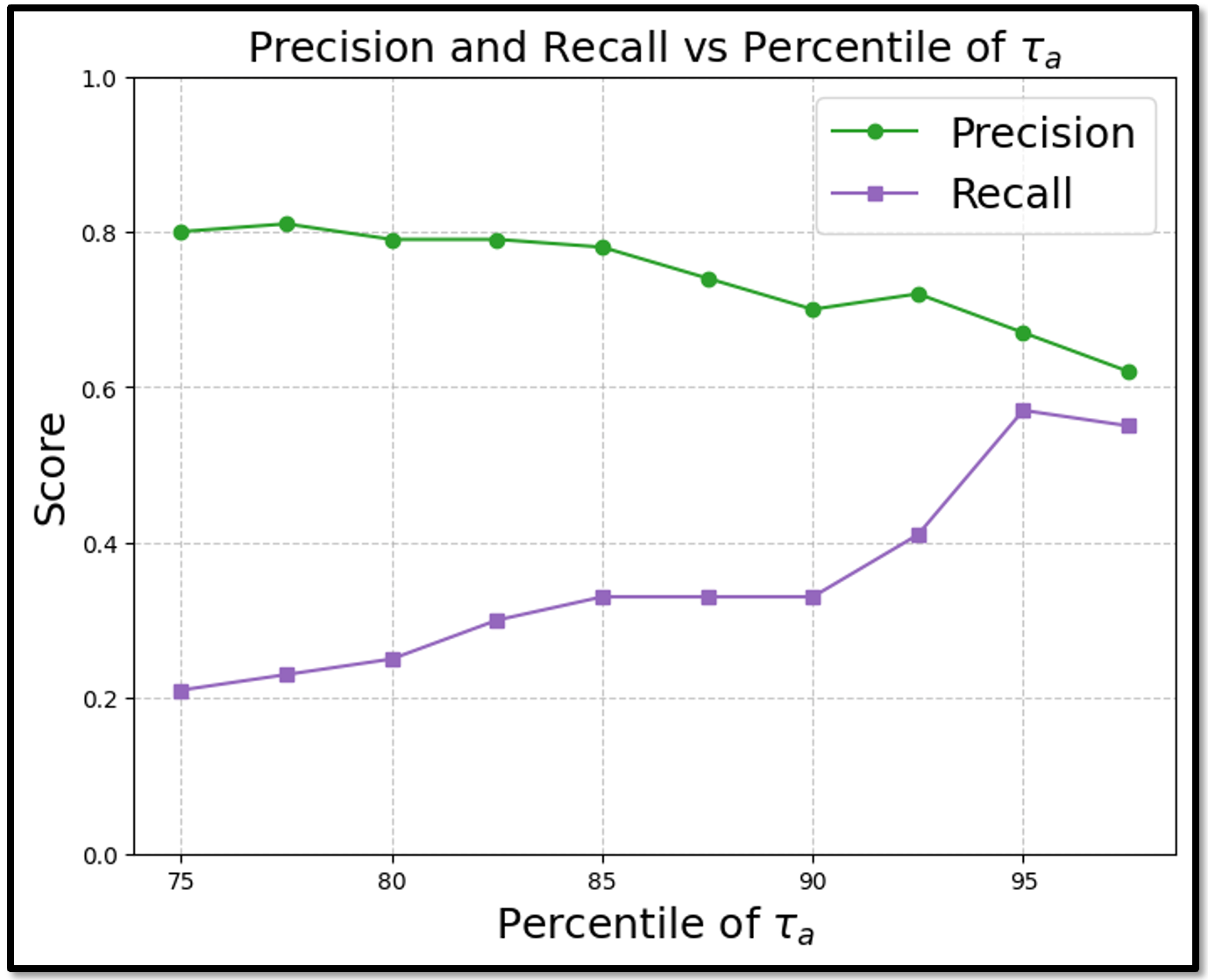}}
  \caption{Sensitivity of $ARL_0$, and $ARL_1$ (left), and Precision-Recall to threshold $\tau_a$}
  \label{fig:synthetic:sensitivity_tau_a} 
\end{figure}

%% file: Appendix/RealWorld.tex
The study employs the public HAI 23.05 cyber-physical system benchmark, which records physical sensor measurements like flow, level, pressure, temperature, vibration, and valve-position signals, etc., from three processes -- Boiler (P1), Turbine (P2), Water Treatment (P3) along with a \textit{Hardware-in-Loop} (HIL) simulation (P4) that introduces dependencies between P1, P2, and P3.

\subsubsection{Pre-processing Steps}
We retain only P1, P2, and P3, thereby removing the HIL simulation (which has only 1 sensor). Furthermore, in our study, we only utilize physical sensor measurements from P1, P2, and P3. Boolean and Integral (or \%) variables arising from control valves (CV) or set points (SP) in any of P1, P2, and P3 are not considered, as they do not directly contribute to the process dynamics. 

To place every variable on a common, dimensionless footing, we center its trajectory by the mean operating level and scale it by the corresponding standard deviation estimated exclusively from the training dataset. The resulting transformation parameters (\(\mu\) and \(\sigma\) pairs) are archived separately for P1, P2, and P3, so that any future segment can be mapped into the very same standardized feature space without recalculating statistics. All training logs are first concatenated and chronologically ordered. For the test data (inference), the stored parameters are applied to every labeled test segment, ensuring normalization and preserving unaltered timestamps.

\subsubsection{Experimental Settings}  
\textbf{Train/Test Split.} Models use the dataset's official partition (nominal operation for training, attacks for testing), trained on the first 10,000 time-steps to capture dependencies while maintaining tractability.

\textbf{Proprietary Client Model.} Each client $c\in\{\mathrm{P1},\mathrm{P2},\mathrm{P3}\}$ trains an EKF with LSTM state transition function $f_m$ (width 16, $P=2$) and MLP head based measurement function $h_m$ (SELU, width 32), optimized via Adam (lr=$10^{-3}$, 10 epochs, $L=20$, batch=64) on MSE. Inference uses $f_m,h_m$ with identity-initialized covariance and $Q,R\propto I$.

\textbf{Augmented Client Model.} Each client adds $\phi_m$ (LSTM, width 1, output $P$) predicting $\delta_t=\phi_m(y_m^t)$ for state correction ${(\hat{x}^{\,t}_{m})}_a \leftarrow {(\hat{x}^{\,t}_{m})}_c + \delta_t$, updated via proprietary model loss and server gradients.

\textbf{Global Server Model.} Server LSTM $f_s$ (width 64) maps $\mathbf{x}^{t-1}\in\mathbb{R}^{MP}$ ($M=3$,$P=2$) to predictions, updated online via Adam (lr=$10^{-3}$) with gradients $\nabla_{{\hat{x}}^{t}_{a}}{L}_s$ for client updates.

\textbf{Centralized Oracle.} EKF that uses LSTM (width 128, $P_{\text{tot}}=6$) for state transition func. and MLP head (SELU, width 4) for measurement func., trained with L=20, batch=64, 10 epochs, lr=$5\times10^{-3}$.

\textbf{Hardware.} PyTorch 3 on CPU-only Google Colab (13GB RAM), without dropout, or grad. clipping.

\begin{figure}
    \centering
  \subfloat{%
       \includegraphics[width=0.48\columnwidth]{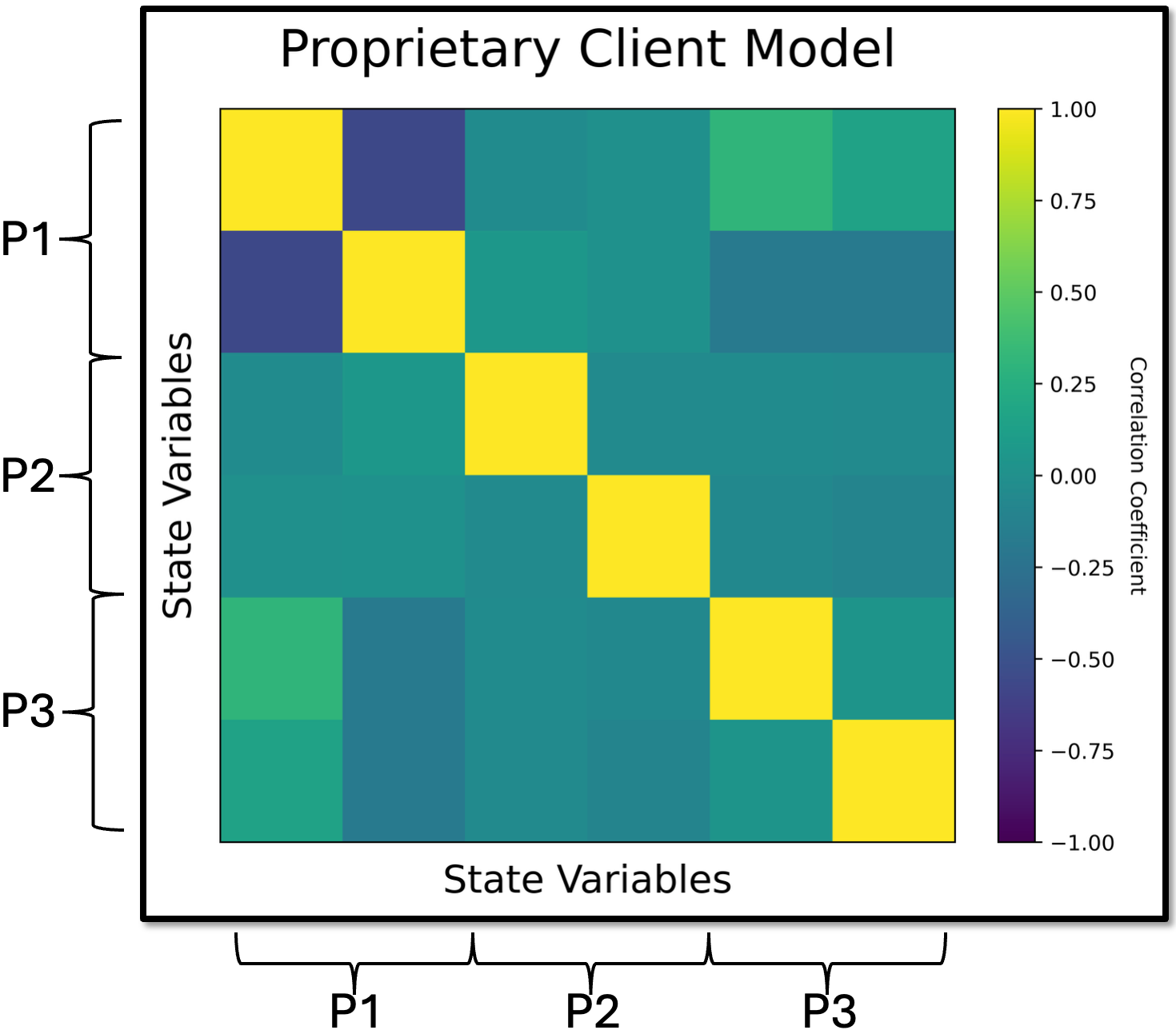}}
      \subfloat{%
        \includegraphics[width=0.48\columnwidth]{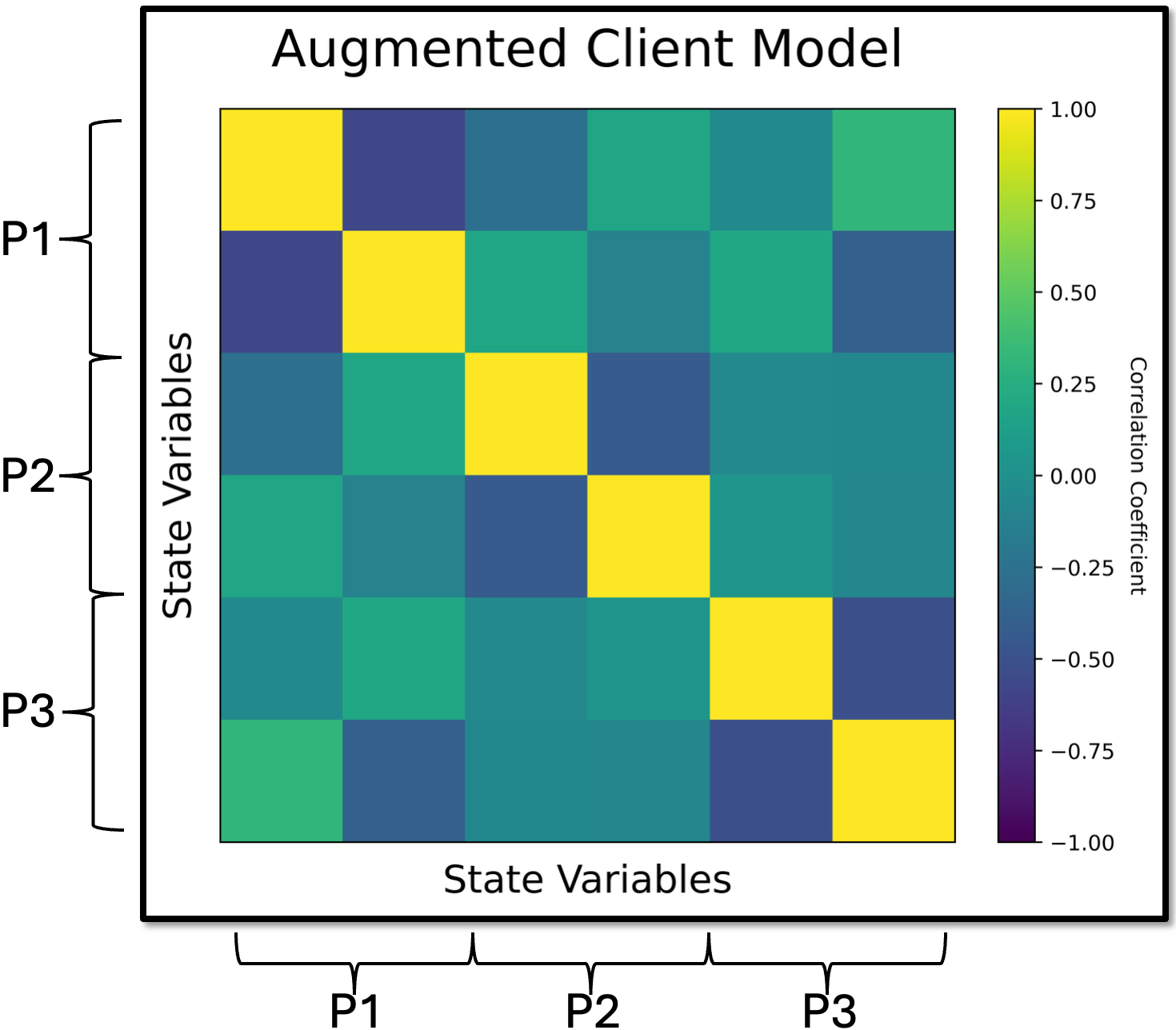}}\\
          \subfloat{%
       \includegraphics[width=0.48\columnwidth]{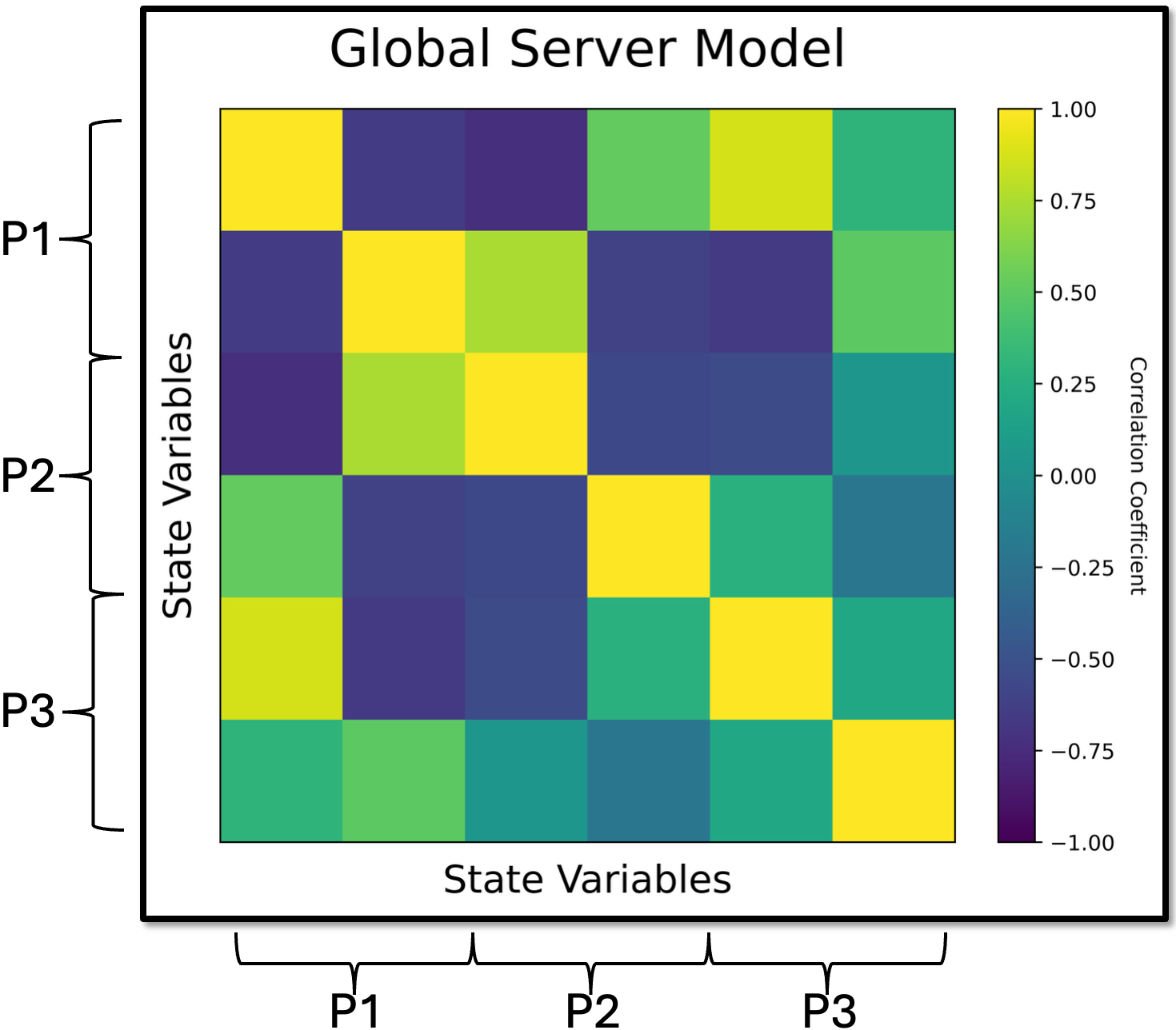}}
      \subfloat{%
        \includegraphics[width=0.48\columnwidth]{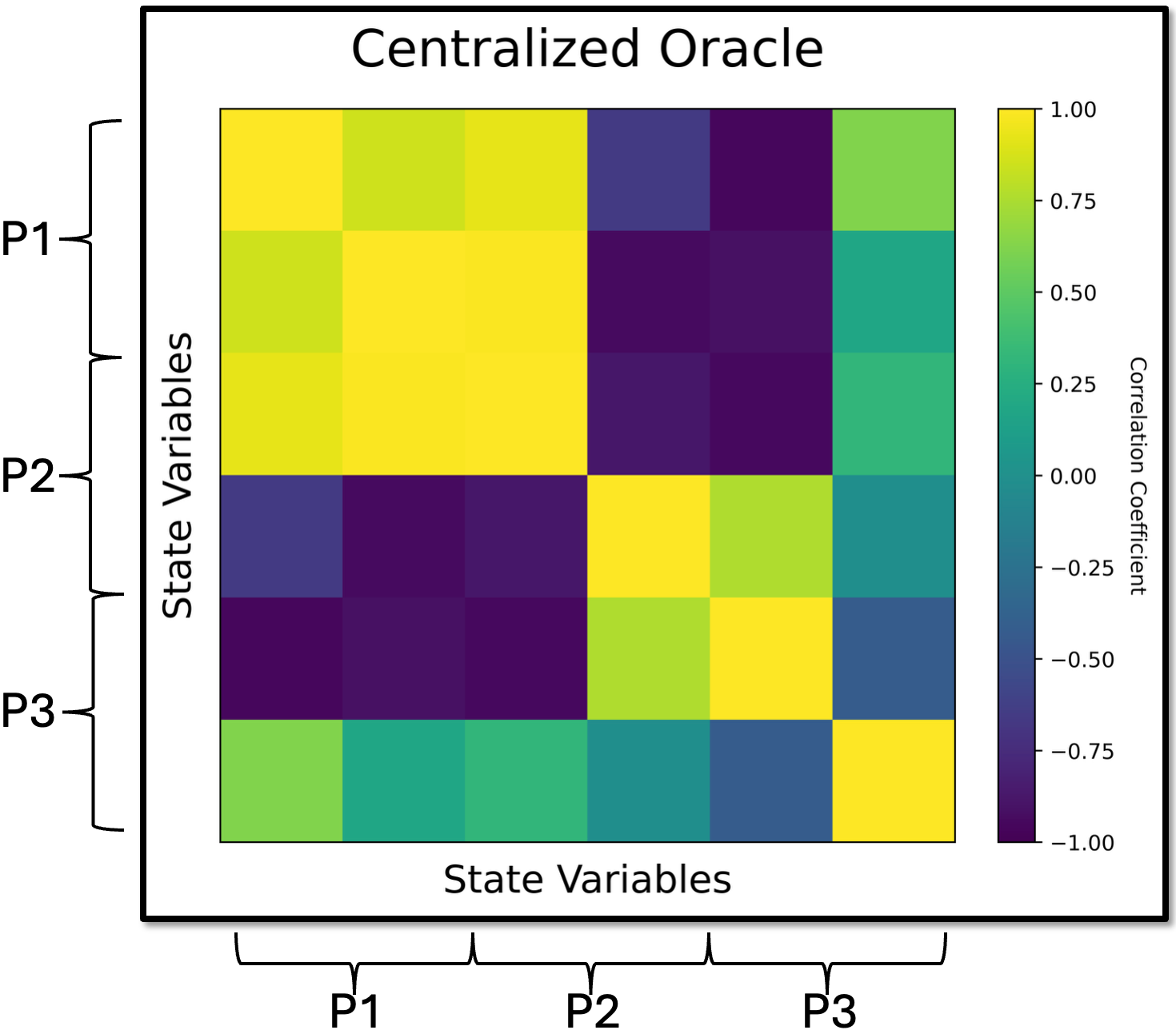}}
  \caption{Correlations between the estimated states of P1, P2, and P3 in proprietary client models (top left), augmented client models (top right), global server model (bottom left), and centralized oracle (bottom right).}
  \label{fig:Correlation_RealWorld} 
\end{figure}

\subsubsection{Training} 
Training on the HAI 23.05 benchmark, we quantify how well each model variant captures the real cross-process couplings among P1, P2, and P3. Fig. ~\ref{fig:Correlation_RealWorld} presents the Pearson correlation matrices of the concatenated states (2 states each per process, from all 3 processes) obtained using the following four models: \textbf{(1)} Proprietary Client Model, \textbf{(2)} Augmented Client Model, \textbf{(3)} Global Server Model, and \textbf{(4)} Centralized Oracle. 

The four correlation maps reveal a clear progression in the ability to capture cross-process structure.  
The \emph{Proprietary Client Model}, trained in isolation, displays bright diagonal blocks but near-zero off-diagonals, signalling that it primarily encodes within-process dynamics while remaining largely blind to interdependencies between \(\mathrm{P2}\) and \(\mathrm{P3}\).  
Augmenting this model with a small network \(\phi_m\) with 1 hidden layer changes this picture markedly: the augmented clients exhibit pronounced positive correlations in the $\mathrm{P2}$ -- $\mathrm{P3}$ and $\mathrm{P1}$-$\mathrm{P2}$ blocks.  

The iterative optimization between the client and the server sharpens the off-diagonal patterns of the \textit{Global Server Model}, and the resulting correlation matrix becomes similar to that of the \textit{Centralized Oracle}. This further indicates that the joint predictor \(f_s\) successfully propagates inter-process information back to every client.  

\begin{figure}[H]
    \centering
  \subfloat{%
       \includegraphics[width=\columnwidth]{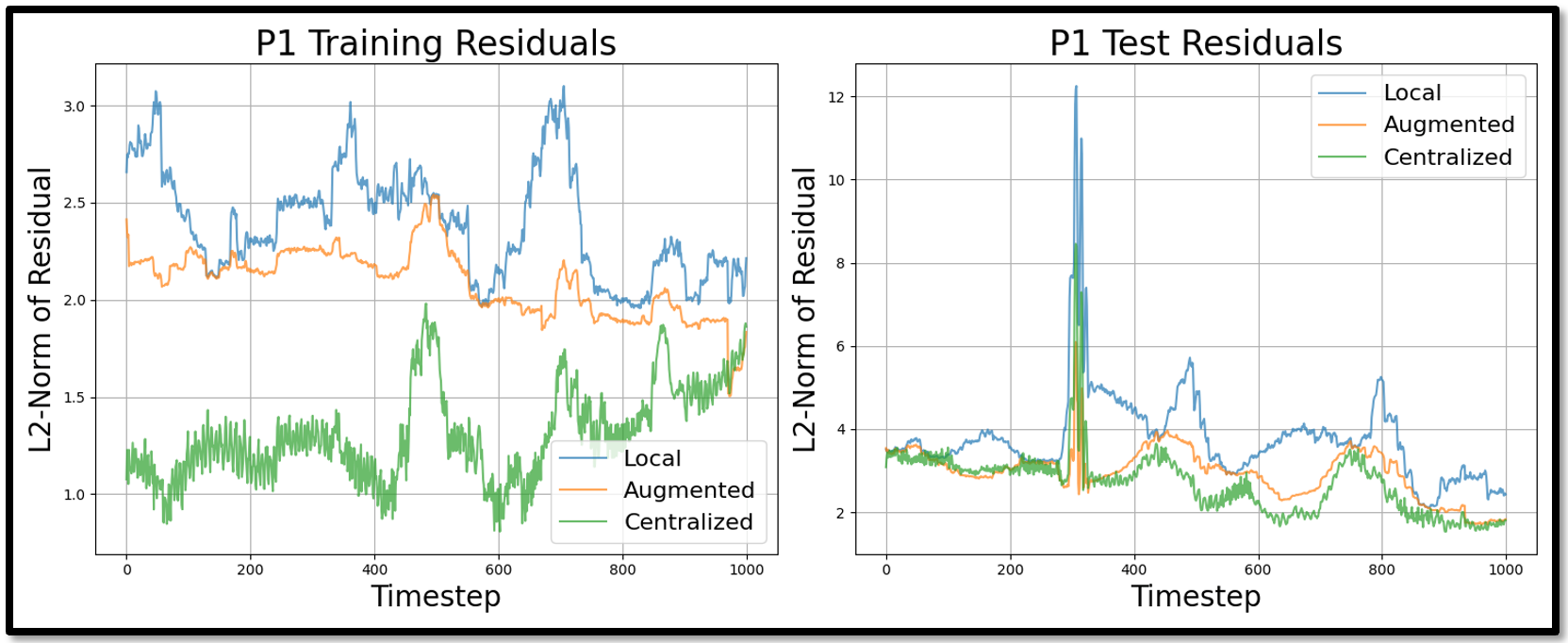}}\\
      \subfloat{%
        \includegraphics[width=\columnwidth]{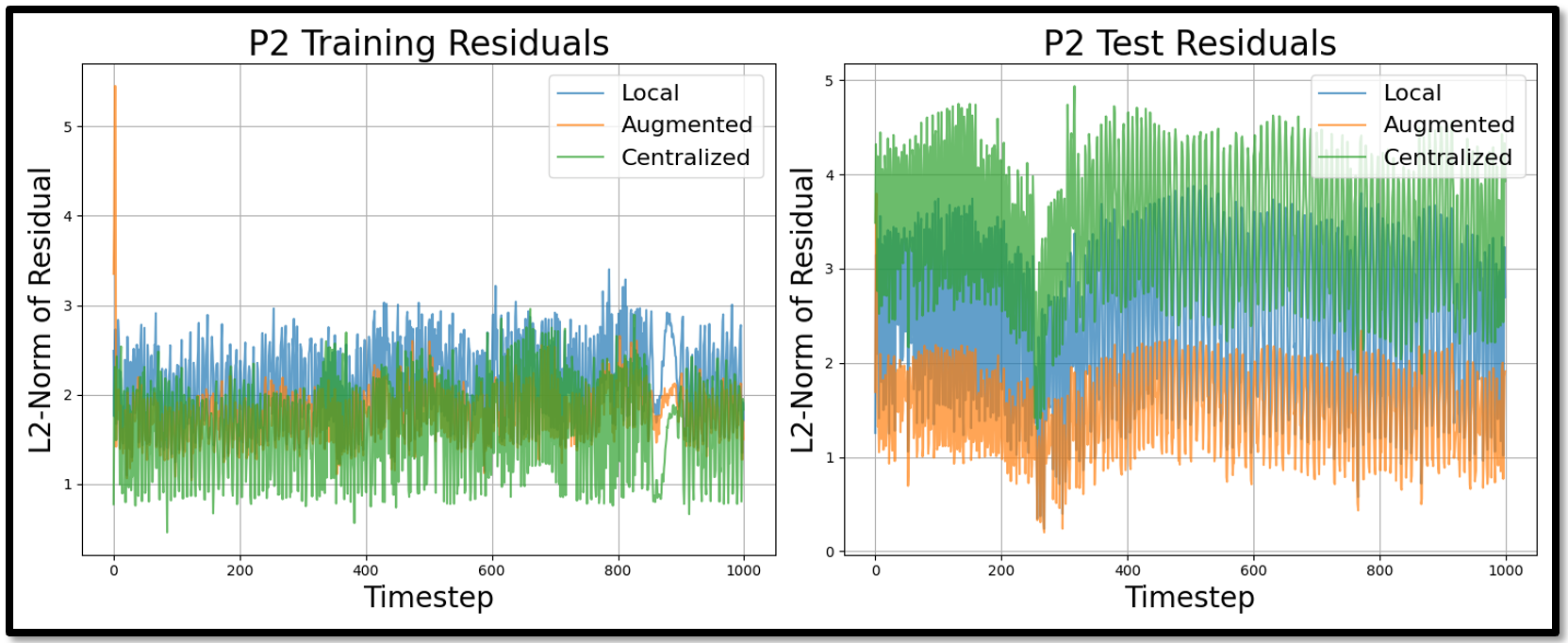}}\\
          \subfloat{%
       \includegraphics[width=\columnwidth]{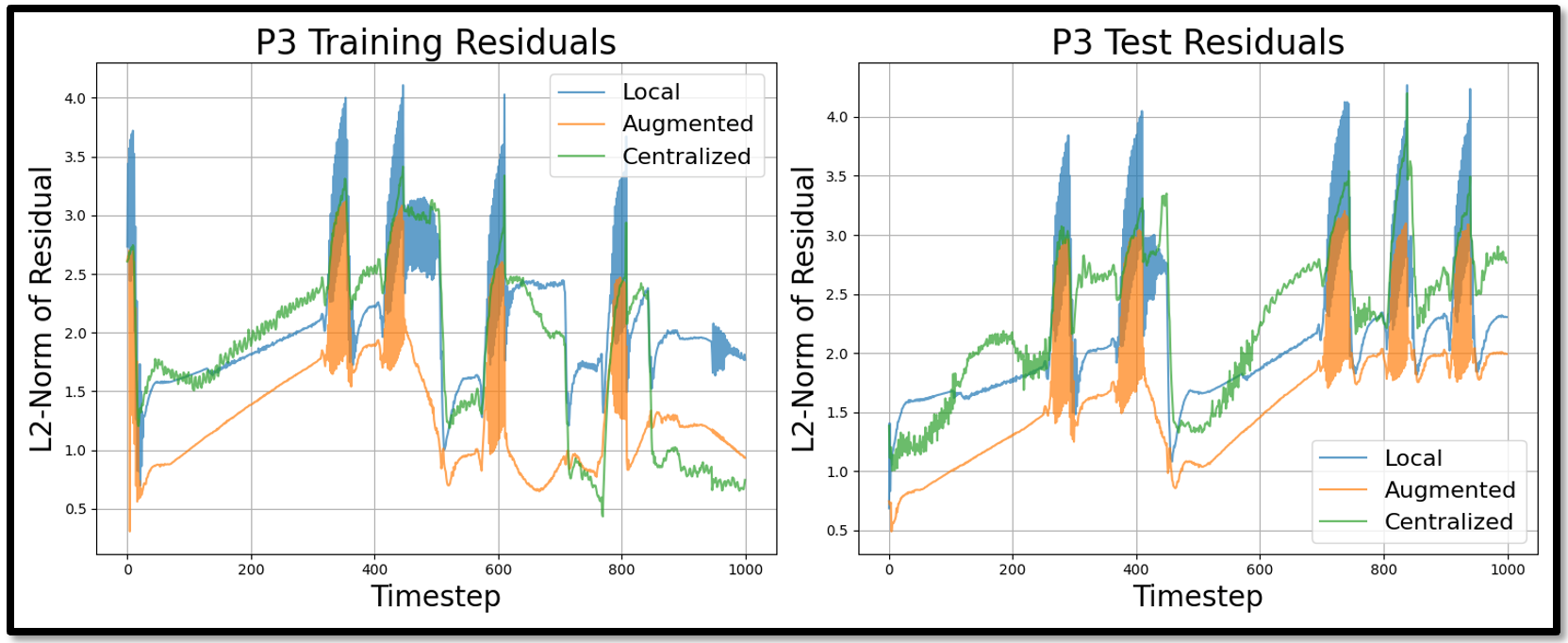}}
  \caption{Training and test residuals for the first 1000 time steps in P1 (top), P2 (middle), and P3 (bottom). Each figure shows the L2 norm of the residuals for the proprietary client model (labeled as ``local''), augmented client model, and the centralized oracle.}
  \label{fig:residuals} 
\end{figure}

\subsubsection{Inference}
\textbf{Residuals.} Fig.~\ref{fig:residuals} presents the $L_2$‐norm residuals (or reconstruction error for high-dimensional data) for P1, P2, P3 over the first 1,000 time‐steps in both training and test cases.  In every case, for most time instants the residual curves obey the ordering,
$\textit{\textbf{  \text{Local (Proprietary) Client Model} \;>\;
  \text{Augmented Client Model} \;>\;
  \text{Centralized Oracle}\,.}}$
During both nominal and attack conditions, the (local) proprietary client model exhibits the largest residuals while the augmented client model consistently lowers the error, approaching the centralized oracle benchmark. This ordering enables us to do systematic RCA using localized residual-based anomaly detection and ensures the validity of Table 1 (see main text of the paper). 

\textbf{Performance Metrics.} We first infer client‐specific Mahalanobis thresholds at selected training‐percentile levels (95\%, 99.5\%), and then, for each threshold, convert every model’s residual series into a Boolean alarm flag per minute.  Using these flags we compute: \textbf{(i)} $ARL_0$, the average time to first false alarm on nominal training data; \textbf{(ii)} $ARL_1$,  over the known attack windows on test data; and \textbf{(iii)} RCA accuracy, and \textbf{(iv)} RCA delay (mean $\pm$ std dev.).  

We limit the inference experiments to \textit{August 12, 2022}, where the details of the attack scenarios including start time and duration is provided. Furthermore, we know that the ground-truth root cause (or location of the attack) is P1.  Table~\ref{tab:performance_summary} presents the aforementioned performance metrics for the Proprietary Client Model, Our Methodology (Augmented Client Model $+$ Global Server Model), and the Centralized Oracle, across the three percentile thresholds for the data on \textit{August 12, 2022}.  
   
    
\begin{table*}
  \centering
  \caption{Summary of ARL$_0$, ARL$_1$, RCA accuracy, and RCA delay (mean\,$\pm$\,std) at different Mahalanobis‐threshold percentiles for each modeling approach.}
  \label{tab:performance_summary}
  \begin{tabular}{clcccc}
    \toprule
    \textbf{Percentile} & \textbf{Model} & \textbf{ARL$_0$} & \textbf{ARL$_1$} & \textbf{RCA Accuracy} & \textbf{RCA Delay in min} \\
     &  & \textbf{(min)} & \textbf{(min)} & \textbf{(\%)} & \textbf{(mean$\pm$std. dev)} \\
    \midrule
   
    {}  & Proprietary Client 
           & 71.00 & 48.50 &  &  \\
    95\%   & \textbf{Our Methodology} 
           & 131.00 & 52.50 & 50\% & 22.00$\pm$11.05 \\
    {}   & Centralized Oracle 
           & 89.00 & 52.50 &  &  \\
    \midrule
    
    {} & Proprietary Client 
           & 71.00 & 48.50 &  &  \\
    99.5\% & \textbf{Our Methodology}
           & 70.66 & 48.50 & 37.5\% & 8.25$\pm$8.63 \\
    {} & Centralized Oracle 
           & 59.66 & 48.50 &  &  \\
    \bottomrule
  \end{tabular}
\end{table*}

As the Mahalanobis‐threshold percentile increases from 95 \% to 99.5 \%, ARL$_0$ for all models rises showing fewer false alarms under nominal conditions. On the other hand, ARL$_1$ remains essentially unchanged, since every attack eventually exceeds even the strictest cutoff.  Our methodology, at 95 \% we observe ARL$_0$ = 131 min (versus 71 min for the proprietary client and 89 min for the oracle), perfect detection speed (ARL$_1$ = 52.5 min), but only 50 \% RCA accuracy with a relatively large attribution delay of 22.0 ± 11.1 min.  When we tighten to 99.5\%, ARL$_0$ falls to 70.7 min, on par with the proprietary filter and close to the centralized benchmark (59.7 min). Furthermore, the attribution delay improves dramatically to 8.3 ± 8.6 min, albeit at the cost of lower RCA accuracy (37.5\%).  This trade‐off shows that raising the threshold reduces false positives and speeds up root‐cause analysis, but can also miss correct attributions, allowing practitioners to tune between accuracy and timeliness.

%% file: Appendix/proofs.tex
\subsection{Inferencing: Root Cause Analysis}
\subsubsection{Proposition 5.2}
\begin{proof}
We know that primary anomalies at each client follow a Poisson arrival process (assumption). For a single client \( C_m \), the probability of an anomaly occurring in \( \Delta t \) is proportional to \( \lambda_m \Delta t \), where \( \lambda_m \) is the Poisson rate parameter for \( C_m \). Furtheremore, the arrival processes for primary anomalies are independent across clients.

Probability of Multiple Anomalies:
    \begin{itemize}
        \item The probability of exactly one client experiencing an anomaly in \( \Delta t \) is:
        \[
        P\left(\sum_{m=1}^M N_m(\Delta t) = 1\right) = \sum_{m=1}^M \lambda_m \Delta t + \mathcal{O}(\Delta t^2).
        \]
        \item The probability of two or more clients experiencing anomalies in \( \Delta t \) is:
        \[
        P\left(\sum_{m=1}^M N_m(\Delta t) > 1\right) = \mathcal{O}(\Delta t^2).
        \]
    \end{itemize}

 As \( \Delta t \to 0 \), the second-order term \( \mathcal{O}(\Delta t^2) \) vanishes faster than the first-order term, making the probability of multiple clients experiencing anomalies negligible:
    \[
    \lim_{\Delta t \to 0} P\left(\sum_{m=1}^M N_m(\Delta t) > 1\right) = 0.
    \]
\end{proof}
\subsection{Convergence Analysis}
\subsubsection{Lemma 6.5}
\begin{proof}
We want to bound
\begin{equation*}
\mathbb{E}\Bigl\|\phi_m\bigl(y_m^{t-1}; \theta_m^k\bigr) 
\;-\; 
\sum_{n \neq m} \mathbb{E}\bigl(\hat{x}_n^{t-1}\bigr)_o\Bigr\|.
\end{equation*}
Adding and subtracting the term $\phi_m\bigl(y_m^{t-1}; \theta_m^*\bigr)$ inside the norm we obtain
\begin{align*}
&\phi_m\bigl(y_m^{t-1}; \theta_m^k\bigr) 
\;-\;
\sum_{n \neq m} \mathbb{E}\bigl(\hat{x}_n^{t-1}\bigr)_o \\
&\qquad=
\underbrace{\Bigl[\phi_m\bigl(y_m^{t-1}; \theta_m^k\bigr)
- 
\phi_m\bigl(y_m^{t-1}; \theta_m^*\bigr)\Bigr]}_{\text{(i)}}
\;+\;
\underbrace{\Bigl[\phi_m\bigl(y_m^{t-1}; \theta_m^*\bigr)
- 
\sum_{n \neq m} \mathbb{E}\bigl(\hat{x}_n^{t-1}\bigr)_o\Bigr]}_{\text{(ii)}}.
\end{align*}
By the triangle inequality,
\begin{equation*}
\mathbb{E}\Bigl\|\phi_m\bigl(y_m^{t-1}; \theta_m^k\bigr)
- 
\sum_{n \neq m} \mathbb{E}\bigl(\hat{x}_n^{t-1}\bigr)_o\Bigr\|
\;\;\le\;\;
\|\text{(i)}\|\;+\;\|\text{(ii)}\|.
\end{equation*}

\noindent
\textbf{Term (ii).} By definition of the irreducible error (Definition 6.4),
\begin{equation*}
\|\text{(ii)}\|
=
\mathbb{E}\Bigl\|\phi_m\bigl(y_m^{t-1}; \theta_m^*\bigr)
- 
\sum_{n \neq m} \mathbb{E}\bigl(\hat{x}_n^{t-1}\bigr)_o\Bigr\|
=
\epsilon^*.
\end{equation*}

\noindent
\textbf{Term (i).} By the assumption that $\phi_m(\cdot;\theta_m)$ is $\mathcal{L}_{\theta_m}$-Lipschitz in~$\theta_m$, we have
\begin{equation*}
\mathbb{E}\Bigl\|\phi_m\bigl(y_m^{t-1}; \theta_m^k\bigr)
- 
\phi_m\bigl(y_m^{t-1}; \theta_m^*\bigr)\Bigr\|
\;\;\le\;\;
\mathcal{L}_{\theta_m}\,\bigl\|\theta_m^k - \theta_m^*\bigr\|.
\end{equation*}
Combining the two terms,
\begin{equation*}
\mathbb{E}\Bigl\|\phi_m\bigl(y_m^{t-1}; \theta_m^k\bigr)
- 
\sum_{n \neq m} \mathbb{E}\bigl(\hat{x}_n^{t-1}\bigr)_o\Bigr\|
\;\;\le\;\;
\mathcal{L}_{\theta_m}\,\bigl\|\theta_m^k - \theta_m^*\bigr\|
\;+\;
\epsilon^*.
\end{equation*}

\noindent
\textbf{Convergence.} 
Under the assumptions that the augmented client loss ${(L_m)}_a$ is convex and $\mathcal{L}_{L_m}$-Lipschitz smooth, and with learning rates $\eta_1,\eta_2 \le 1/L$, we have $\theta_m^k \to \theta_m^*$ as $k\to\infty$.  Thus,
\begin{equation*}
\lim_{k\to\infty}
\mathbb{E}\Bigl\|\phi_m\bigl(y_m^{t-1}; \theta_m^k\bigr)
- 
\sum_{n \neq m} \mathbb{E}\bigl(\hat{x}_n^{t-1}\bigr)_o\Bigr\|
\;\;\le\;\;
\lim_{k\to\infty}
\Bigl(
\mathcal{L}_{\theta_m}\,\bigl\|\theta_m^k - \theta_m^*\bigr\| 
\;+\;
\epsilon^*
\Bigr)
\;=\;
\epsilon^*.
\end{equation*}
Hence, the limit of the norm difference is bounded by~$\epsilon^*$. 
\end{proof}
\subsubsection{Proposition 6.8}
\begin{proof}
We know that the Jacobians of the state transitions for both the local client model and the centralized oracle are bounded. Let \( f_m \) represent the state transition function of client \( m \), and \( f \) represent the state transition function of the centralized oracle.

For client \( m \), let \( z_m \) be the input to its local state transition function, and let \( \mathbf{z} \) denote the combined state across all clients (input to the centralized oracle). The extraction function \( \text{Extract}_m \) maps the centralized state transition output to the corresponding client \( m \)'s state.

The structural difference is defined as:
\[
\mathbb{E}\|f_m(z_m) - \text{Extract}_m(f(\mathbf{z}))\|.
\]

Expanding \( f_m(z_m) \) and \( \text{Extract}_m(f(\mathbf{z})) \), we note that the boundedness of the Jacobians implies that there exist constants \( \mathcal{L}_{f_m} \) and \( \mathcal{L}_f \), representing the Lipschitz constants of \( f_m \) and \( f \) respectively, such that:
\[
\|f_m(z_m) - f_m(z_m')\| \leq \mathcal{L}_{f_m} \|z_m - z_m'\| \quad \text{and} \quad \|f(\mathbf{z}) - f(\mathbf{z}')\| \leq \mathcal{L}_f \|\mathbf{z} - \mathbf{z}'\|.
\]

Furthermore, the mapping \( \text{Extract}_m \) is linear and bounded, meaning there exists a constant \( C_m \) such that:
\[
\|\text{Extract}_m(f(\mathbf{z})) - \text{Extract}_m(f(\mathbf{z}'))\| \leq C_m \|f(\mathbf{z}) - f(\mathbf{z}')\|.
\]

Now, consider the difference:
\[
\mathbb{E}\|f_m(z_m) - \text{Extract}_m(f(\mathbf{z}))\|.
\]

Using the triangle inequality:
\[
\mathbb{E}\|f_m(z_m) - \text{Extract}_m(f(\mathbf{z}))\| \leq \mathbb{E}\|f_m(z_m) - f_m(z_m')\| + \mathbb{E}\|f_m(z_m') - \text{Extract}_m(f(\mathbf{z}))\|.
\]

By the Lipschitz property of \( f_m \):
\[
\mathbb{E}\|f_m(z_m) - f_m(z_m')\| \leq \mathcal{L}_{f_m} \|z_m - z_m'\|.
\]

For the second term, using the boundedness of \( \text{Extract}_m \) and the Lipschitz property of \( f \):
\[
\mathbb{E}\|f_m(z_m') - \text{Extract}_m(f(\mathbf{z}))\| \leq C_m \mathcal{L}_f \|\mathbf{z} - \mathbf{z}'\|.
\]

Combining these results:
\[
\mathbb{E}\|f_m(z_m) - \text{Extract}_m(f(\mathbf{z}))\| \leq \mathcal{L}_{f_m} \|z_m - z_m'\| + C_m \mathcal{L}_f \|\mathbf{z} - \mathbf{z}'\|.
\]

Since \( z_m \), \( \mathbf{z} \), and their respective transformations are bounded, there exists a constant \( \delta_m \) such that:
\[
\mathbb{E}\|f_m(z_m) - \text{Extract}_m(f(\mathbf{z}))\| \leq \delta_m.
\]
\end{proof}
\subsubsection{Theorem 6.10}
\begin{proof}
The augmented client model predicts the state as:
\[
(x^{t, k}_m)_a = f_m\Bigl(\mathbb{E}\bigl[(\hat{x}^{t-1}_m)_c\bigr] + \phi_m\bigl(y^{t-1}_m; \theta_m^k\bigr)\Bigr),
\]
while the centralized oracle predicts the state using information from all clients:
\[
(x^t_m)_o = f_m\Bigl(\mathbb{E}\bigl[(\hat{x}^{t-1}_m)_o\bigr], \{\mathbb{E}\bigl[(\hat{x}^{t-1}_n)_o\bigr]\}_{n \neq m}\Bigr).
\]

The difference between the predicted states is:
\[
(x^{t, k}_m)_a - (x^t_m)_o = f_m\Bigl(\mathbb{E}\bigl[(\hat{x}^{t-1}_m)_c\bigr] + \phi_m\bigl(y^{t-1}_m; \theta_m^k\bigr)\Bigr)
- \text{Extract}_m\Biggl(f_m\Bigl(\mathbb{E}\bigl[(\hat{x}^{t-1}_m)_o\bigr], \{\mathbb{E}\bigl[(\hat{x}^{t-1}_n)_o\bigr]\}_{n \neq m}\Bigr)\Biggr).
\]

Using the Lipschitz continuity of \(f_m(\cdot)\), we have:
\[
\|f_m(\mathbf{x}_1) - f_m(\mathbf{x}_2)\| 
\leq \mathcal{L}_{f_m} \|\mathbf{x}_1 - \mathbf{x}_2\|,
\quad \forall\, \mathbf{x}_1, \mathbf{x}_2.
\]

Substituting:
\[
\mathbf{x}_1 = \mathbb{E}\bigl[(\hat{x}^{t-1}_m)_c\bigr] + \phi_m\bigl(y^{t-1}_m; \theta_m^k\bigr), 
\quad
\mathbf{x}_2 = \Bigl(\mathbb{E}\bigl[(\hat{x}^{t-1}_m)_o\bigr], \{\mathbb{E}\bigl[(\hat{x}^{t-1}_n)_o\bigr]\}_{n \neq m}\Bigr),
\]
we obtain:
\begin{equation*}
\begin{split}
\Bigl\| f_m\Bigl(\mathbb{E}\bigl[(\hat{x}^{t-1}_m)_c\bigr] + \phi_m\bigl(y^{t-1}_m; \theta_m^k\bigr)\Bigr)
\;-\; f_m\Bigl(\mathbb{E}\bigl[(\hat{x}^{t-1}_m)_o\bigr] + \phi_m\bigl(y^{t-1}_m; \theta_m^k\bigr)\Bigr) \Bigr\|
\;&\le\;
\mathcal{L}_{f_m}\Bigl(
\bigl\|\mathbb{E}\bigl[(\hat{x}^{t-1}_m)_c\bigr] - \mathbb{E}\bigl[(\hat{x}^{t-1}_m)_o\bigr]\bigr\|
\;\\&+\;
\bigl\|\phi_m\bigl(y^{t-1}_m; \theta_m^k\bigr)
-\! \sum_{n \neq m}\mathbb{E}\bigl[(\hat{x}_n^{t-1})_o\bigr]\bigr\|
\Bigr).
\end{split}
\end{equation*}

We also know from Proposition 6.8 that:
\[
\Biggl\|f_m\Bigl(\mathbb{E}\bigl[(\hat{x}^{t-1}_m)_o\bigr] + \phi_m\bigl(y^{t-1}_m; \theta_m^k\bigr)\Bigr)
-\text{Extract}_m\Biggl(f\Bigl(\mathbb{E}\bigl[(\hat{x}^{t-1}_m)_o\bigr], \{\mathbb{E}\bigl[(\hat{x}^{t-1}_n)_o\bigr]\}_{n \neq m}\Bigr)\Biggr)\Biggr\| \leq \delta_m.
\]

Thus, we obtain the following upper bound:
\begin{equation*}
\Bigl\|(x^{t, k}_m)_a - (x^t_m)_o\Bigr\| 
\;\le
\mathcal{L}_{f_m}\Bigl(
\bigl\|\mathbb{E}\bigl[(\hat{x}^{t-1}_m)_c\bigr] - \mathbb{E}\bigl[(\hat{x}^{t-1}_m)_o\bigr]\bigr\|
\;+\;
\bigl\|\phi_m\bigl(y^{t-1}_m; \theta_m^k\bigr)
-\! \sum_{n \neq m}\mathbb{E}\bigl[(\hat{x}_n^{t-1})_o\bigr]\bigr\|
\Bigr) + \delta_m.
\end{equation*}

Next, we apply Lemma 6.5. Recall:
\[
\bigl\|\phi_m\bigl(y^{t-1}_m; \theta_m^k\bigr) 
-\!\!\!
\sum_{n \neq m} \mathbb{E}\bigl[(\hat{x}_n^{t-1})_o\bigr]\bigr\|
\;\le\;
\mathcal{L}_{\theta_m}\,\|\theta_m^k - \theta_m^*\|
\;+\;
\epsilon^*,
\]
where \(\epsilon^*\) is the irreducible error. Substituting the lemma into the first inequality, we have:
\begin{align*}
\|(x^{t,k}_m)_a - (x^t_m)_o\|
\;\le\;
\mathcal{L}_{f_m} \Bigl(
\bigl\|\mathbb{E}\bigl[(\hat{x}^{t-1}_m)_c\bigr] - \mathbb{E}\bigl[(\hat{x}^{t-1}_m)_o\bigr]\bigr\| 
+ \mathcal{L}_{\theta_m}\|\theta_m^k - \theta_m^*\|
+ \epsilon^*
\Bigr) + \delta_m.
\end{align*}

Focusing on \(\bigl\|\mathbb{E}\bigl[(\hat{x}^t_m)_c\bigr] - \mathbb{E}\bigl[(\hat{x}^t_m)_o\bigr]\bigr\|\), and using the structure of the Kalman filter estimates and the Lipschitz continuity of \(h_m(\cdot)\), we derive the recursion:
\[
D_t := \bigl\|\mathbb{E}\bigl[(\hat{x}^t_m)_c\bigr] - \mathbb{E}\bigl[(\hat{x}^t_m)_o\bigr]\bigr\| 
\leq \|\Delta K_m\| \cdot \mathbb{E}\bigl[\|{(r_m^t)}_o\|\bigr] 
+ \mathcal{L}_{h_m} C_x 
+ (1 + \|\Delta K_m\| \mathcal{L}_{h_m}) \Bigl(\mathcal{L}_{f_m} D_{t-1} + \mathcal{L}_{f_m} (M-1) C_x\Bigr),
\]
where \({(r_m^t)}_o := y_m^t - h_m((x^t_m)_o)\).

Expanding the recursion iteratively and assuming the stability condition \(1 + \|\Delta K_m\| \mathcal{L}_{h_m} < \frac{1}{\mathcal{L}_{f_m}}\), the recursive terms decay geometrically. The steady-state error \(D_t\) is bounded by:
\[
\lim_{t \to \infty} D_t 
\leq \frac{\|\Delta K_m\| \sup_t \mathbb{E}\bigl[\|{(r_m^t)}_o\|\bigr] + \mathcal{L}_{h_m} C_x + \mathcal{L}_{f_m} (M-1) C_x}{1 - (1 + \|\Delta K_m\| \mathcal{L}_{h_m}) \mathcal{L}_{f_m}}.
\]

Substituting this result into the bound for \(\|(x^{t, k}_m)_a - (x^t_m)_o\|\), and noting that \(\|\theta_m^k - \theta_m^*\| \to 0\) as \(k \to \infty\), we obtain:
\[
\lim_{k \to \infty} \|(x^{t, k}_m)_a - (x^t_m)_o\| 
\leq \mathcal{L}_{f_m} \Bigl(\rho + \epsilon^*\Bigr) + \delta_m,
\]
where:
\[
\rho := \frac{\|\Delta K_m\| \sup_t \mathbb{E}\bigl[\|{(r_m^t)}_o\|\bigr] + \mathcal{L}_{h_m} C_x + \mathcal{L}_{f_m} (M-1) C_x}{1 - (1 + \|\Delta K_m\| \mathcal{L}_{h_m}) \mathcal{L}_{f_m}}.
\]

Thus, the error converges as claimed.
\end{proof}
\subsubsection{Theorem 6.13}
\begin{proof}
We aim to bound the error \(\mathbb{E}\bigl[\|x_s^{t, k} - x_o^t\|\bigr]\), where:
\[
x_s^{t, k} = f_s\bigl(\{\mathbb{E}\bigl[(\hat{x}_m^t)_c\bigr]\}_{m=1}^M; \theta_s^k\bigr), \quad x_o^t = f\bigl(\{\mathbb{E}\bigl[(\hat{x}_m^t)_o\bigr]\}_{m=1}^M\bigr).
\]

Re-writing the difference, we obtain:
\[
\|x_s^{t, k} - x_o^t\| = \bigl\|f_s\bigl(\{\mathbb{E}\bigl[(\hat{x}_m^t)_c\bigr]\}_{m=1}^M; \theta_s^k\bigr) - f_s\bigl(\{\mathbb{E}\bigl[(\hat{x}_m^t)_o\bigr]\}_{m=1}^M; \theta_s^k\bigr) 
+ f_s\bigl(\{\mathbb{E}\bigl[(\hat{x}_m^t)_o\bigr]\}_{m=1}^M; \theta_s^k\bigr) - f\bigl(\{\mathbb{E}\bigl[(\hat{x}_m^t)_o\bigr]\}_{m=1}^M\bigr)\bigr\|.
\]

Using the triangle inequality, this is bounded as:
\[
\|x_s^{t, k} - x_o^t\| \leq \bigl\|f_s\bigl(\{\mathbb{E}\bigl[(\hat{x}_m^t)_c\bigr]\}_{m=1}^M; \theta_s^k\bigr) - f_s\bigl(\{\mathbb{E}\bigl[(\hat{x}_m^t)_o\bigr]\}_{m=1}^M; \theta_s^k\bigr)\bigr\| + \bigl\|f_s\bigl(\{\mathbb{E}\bigl[(\hat{x}_m^t)_o\bigr]\}_{m=1}^M; \theta_s^k\bigr) - f\bigl(\{\mathbb{E}\bigl[(\hat{x}_m^t)_o\bigr]\}_{m=1}^M\bigr)\bigr\|.
\]

Using Lipschitz continuity of \(f_s\) with respect to its inputs, with constant \(\mathcal{L}_{f_s}\), we have:
\[
\bigl\|f_s\bigl(\{\mathbb{E}\bigl[(\hat{x}_m^t)_c\bigr]\}_{m=1}^M; \theta_s^k\bigr) - f_s\bigl(\{\mathbb{E}\bigl[(\hat{x}_m^t)_o\bigr]\}_{m=1}^M; \theta_s^k\bigr)\bigr\| \leq \mathcal{L}_{f_s} \sum_{m=1}^M \bigl\|\mathbb{E}\bigl[(\hat{x}_m^t)_c\bigr] - \mathbb{E}\bigl[(\hat{x}_m^t)_o\bigr]\bigr\|.
\]

Let \(D_t := \bigl\|\mathbb{E}\bigl[(\hat{x}_m^t)_c\bigr] - \mathbb{E}\bigl[(\hat{x}_m^t)_o\bigr]\bigr\|\). Then:
\[
\bigl\|f_s\bigl(\{\mathbb{E}\bigl[(\hat{x}_m^t)_c\bigr]\}_{m=1}^M; \theta_s^k\bigr) - f_s\bigl(\{\mathbb{E}\bigl[(\hat{x}_m^t)_o\bigr]\}_{m=1}^M; \theta_s^k\bigr)\bigr\| \leq \mathcal{L}_{f_s} \cdot M \cdot \sup_{m} D_t.
\]

In the steady state (large \(t\)), using Theorem 6.10, \(D_t\) satisfies:
\[
D_t \leq \frac{\|\Delta K_m\| \sup_t \mathbb{E}\bigl[\|{(r_m^t)}_o\|\bigr] + \mathcal{L}_{h_m} C_x + \mathcal{L}_{f_m} (M-1) C_x}{1 - (1 + \|\Delta K_m\| \mathcal{L}_{h_m}) \mathcal{L}_{f_m}}.
\]

Using Definition 6.11:
\[
\lim_{k \to \infty} \bigl\|f_s\bigl(\{\mathbb{E}\bigl[(\hat{x}_m^t)_o\bigr]\}_{m=1}^M; \theta_s^k\bigr) - f\bigl(\{\mathbb{E}\bigl[(\hat{x}_m^t)_o\bigr]\}_{m=1}^M\bigr)\bigr\| = \sigma^*.
\]

Substituting:
\[
\lim_{k \to \infty} \mathbb{E}\bigl[\|x_s^{t, k} - x_o^t\|\bigr] 
\leq \mathcal{L}_{f_s} \cdot M \cdot \rho + \sigma^*,
\]
where:
\[
\rho := \frac{\|\Delta K_m\| \sup_t \mathbb{E}\bigl[\|{(r_m^t)}_o\|\bigr] + \mathcal{L}_{h_m} C_x + \mathcal{L}_{f_m} (M-1) C_x}{1 - (1 + \|\Delta K_m\| \mathcal{L}_{h_m}) \mathcal{L}_{f_m}}.
\]
\end{proof}
\subsection{Privacy Analysis}
We provide two definitions that will be used to prove Theorems 7.1 and 7.2. 
\begin{definition}[\textbf{$\ell_2$-Sensitivity}]
The $\ell_2$-sensitivity $\Delta$ of a function $f: \mathcal{D} \rightarrow \mathbb{R}^k$ is the maximum change in the output's $\ell_2$-norm due to a change in a single data point:
\begin{equation*}
\Delta = \max_{D, D'} \| f(D) - f(D') \|_2,
\end{equation*}
where $D$ and $D'$ are neighboring datasets.
\end{definition}

\begin{definition}[\textbf{Gaussian Mechanism}]
Given a function $f: \mathcal{D} \rightarrow \mathbb{R}^k$ with $\ell_2$-sensitivity $\Delta$, the Gaussian mechanism $\mathcal{M}$ adds noise drawn from a Gaussian distribution to each output component:
\begin{equation*}
\mathcal{M}(D) = f(D) + \mathcal{N}(0, \sigma^2 I_k),
\end{equation*}
where $\sigma \geq \frac{\Delta \sqrt{2 \ln (1.25/\delta)}}{\varepsilon}$ ensures that $\mathcal{M}$ satisfies $(\varepsilon, \delta)$-differential privacy.
\end{definition}
\subsubsection{Theorem 7.1}
The local and augmented states are perturbed as:
\[
{(\tilde{x}_{m}^{t-1})}_c = {(\hat{x}_{m}^{t-1})}_c + \mathcal{N}(0, \sigma_c^2),
\quad
{(\tilde{x}_{m}^t)}_a = {(x_{m}^t)}_a + \mathcal{N}(0, \sigma_a^2).
\]

Substituting the Extended Kalman Filter equations, the local state estimate \({(\hat{x}_{m}^{t-1})}_c\) is:
    \[
    {(\hat{x}_{m}^{t-1})}_c = f_m\big({(\hat{x}_{m}^{t-2})}_c\big) + K_m^c \big(y_m^{t-1} - h_m\big({(x_m^{t-1})}_c\big)\big).
    \]
Adding bounds on the Kalman gain \( K_m^c \) and observation Jacobian \( H_m \), the sensitivity is:
    \[
    \Delta_c \leq \|\Delta K_m\| (\mathcal{L}_{h_m} C_x + \sup_t \|r_m^t\|),
    \]
    where \( \mathcal{L}_{h_m} \) and \( C_x \) are bounds on \( h_m \) and state space, respectively, and \( r_m^t \) is the residual.

Similarly, substituting the augmented state equation:
    \[
    {(x_{m}^t)}_a = f_m\big({(\hat{x}_{m}^{t-1})}_a\big) + \phi_m(y_m^t; \theta_m),
    \]
    and bounding \( f_m \) and \(\phi_m\), we obtain:
    \[
    \Delta_a \leq \mathcal{L}_{f_m} C_x + \mathcal{L}_{\theta_m} \sup_t \|y_m^t\|.
    \]

Adding Gaussian noise ensures \((\varepsilon, \delta)\)-differential privacy. Substituting the sensitivity terms into the noise scale equation:
\[
\sigma_c \geq \frac{2 \|\Delta K_m\| (\mathcal{L}_{h_m} C_x + \sup_t \|r_m^t\|) \sqrt{2 \ln(1.25 / \delta_c)}}{\varepsilon_c},
\quad
\sigma_a \geq \frac{2 (\mathcal{L}_{f_m} C_x + \mathcal{L}_{\theta_m} \sup_t \|y_m^t\|) \sqrt{2 \ln(1.25 / \delta_a)}}{\varepsilon_a}.
\]

Adding the privacy budgets of the two mechanisms yields:
\[
\varepsilon = \varepsilon_c + \varepsilon_a, \quad \delta = \delta_c + \delta_a.
\]
\subsubsection{Theorem 7.2}
The server sends perturbed gradients:
\[
\tilde{g}_m^t = \text{clip}(g_m^t, C_g) + \mathcal{N}(0, \sigma_g^2),
\]
where \( \text{clip}(g_m^t, C_g) \) ensures \(\|g_m^t\|_2 \leq C_g\).

Substituting the gradient equation:
    \[
    g_m^t = \nabla_{\theta_m^k} L_s = \big(\nabla_{{(\hat{x}^t_m)}_a} L_s\big) \cdot \big(\nabla_{\theta_m^k} {(\hat{x}^t_m)}_a\big),
    \]
    and bounding each term, we find that the sensitivity is:
    \[
    \Delta_g = \max_{t} \|\text{clip}(g_m^t, C_g)\|_2 \leq C_g.
    \]

Adding Gaussian noise ensures \((\varepsilon, \delta)\)-differential privacy. Substituting the sensitivity bound:
\[
\sigma_g \geq \frac{2 C_g \sqrt{2 \ln(1.25 / \delta)}}{\varepsilon}.
\]

\subsubsection{Theorem 7.3}
To demonstrate that adding \( p_c \) and \( p_a \) satisfies \(\varepsilon\)-differential privacy, we analyze the randomized response mechanism as follows:

Starting with the perturbation mechanism, the binary anomaly flags are perturbed independently using:
\[
\tilde{Z}_c^t =
\begin{cases}
Z_c^t, & \text{with probability } p_c, \\
1 - Z_c^t, & \text{with probability } 1 - p_c,
\end{cases}
\quad
\tilde{Z}_a^t =
\begin{cases}
Z_a^t, & \text{with probability } p_a, \\
1 - Z_a^t, & \text{with probability } 1 - p_a.
\end{cases}
\]

Substituting the definition of \(\varepsilon\)-differential privacy, the mechanism satisfies differential if, for any two neighboring datasets \( \mathcal{D} \) and \( \mathcal{D}' \) differing in one binary flag, and any possible output \( \tilde{z} \):
\[
\frac{\Pr[\mathcal{M}(\mathcal{D}) = \tilde{z}]}{\Pr[\mathcal{M}(\mathcal{D}') = \tilde{z}]} \leq e^{\varepsilon}.
\]

Considering the sensitivity of binary flags, the maximum difference in output caused by changing one record is:
\[
\Delta_Z = \max_{Z, Z'} |Z - Z'| = 1.
\]

Analyzing the probabilities of reporting \( \tilde{Z}_c^t = 1 \) under two neighboring datasets:
\[
\Pr[\tilde{Z}_c^t = 1 \mid Z_c^t = 1] = p_c, \quad \Pr[\tilde{Z}_c^t = 1 \mid Z_c^t = 0] = 1 - p_c.
\]

Dividing these probabilities, the ratio is:
\[
\frac{\Pr[\tilde{Z}_c^t = 1 \mid Z_c^t = 1]}{\Pr[\tilde{Z}_c^t = 1 \mid Z_c^t = 0]} = \frac{p_c}{1 - p_c}.
\]

Substituting the privacy constraint \(\frac{p_c}{1 - p_c} \leq e^{\varepsilon_c}\), solving for \( p_c \):
\[
p_c = \frac{e^{\varepsilon_c}}{1 + e^{\varepsilon_c}}.
\]

Similarly, analyzing the probabilities of reporting \( \tilde{Z}_a^t = 1 \):
\[
\Pr[\tilde{Z}_a^t = 1 \mid Z_a^t = 1] = p_a, \quad \Pr[\tilde{Z}_a^t = 1 \mid Z_a^t = 0] = 1 - p_a,
\]
and substituting \(\frac{p_a}{1 - p_a} \leq e^{\varepsilon_a}\), solving for \( p_a \):
\[
p_a = \frac{e^{\varepsilon_a}}{1 + e^{\varepsilon_a}}.
\]

Adding the privacy budgets of \( \tilde{Z}_c^t \) and \( \tilde{Z}_a^t \), the total privacy budget is:
\[
\varepsilon = \varepsilon_c + \varepsilon_a.
\]

Hence, the randomized response mechanism with probabilities \( p_c \) and \( p_a \) satisfies \(\varepsilon\)-differential privacy.

%% file: Appendix/pseudocode.tex
\subsection{Training: Cross-Client Learning of Interdependencies}
\begin{algorithm}[h]
\caption{Training: Federated Learning of Cross-Client Interdependencies}
\label{alg:training}
\begin{algorithmic}[1]
\STATE \textbf{Input:} Local data \( y^{1:T}_m \), functions \( f_m \), \( h_m \) \( \forall m \in \{1, \ldots, M\} \), learning rates \( \eta_1, \eta_2, \eta_3 \).
\STATE \textbf{Output:} Updated parameters \( \theta_s \), \( \{\theta_m\}_{m=1}^M \).

\FOR{$t = 1:T$}
    \FOR{each \textit{client} \( m \) \textbf{in parallel}}
        \STATE Compute augmented states \( {(\hat{x}^t_m)}_a \) and \( {(x^t_m)}_a \).
        \STATE Send \( {(\hat{x}^{t-1}_m)}_c \) and \( {(x^t_m)}_a \) to the server.
    \ENDFOR
    \STATE \textit{Server} updates \( \theta_s \) using \( \nabla_{\theta_s^k} (L_s) \).
    \STATE \textit{Server} sends \( \nabla_{{(x^t_m)}_a} (L_s) \) to each client.
    \FOR{each \textit{client} \( m \) \textbf{in parallel}}
        \STATE Update \( \theta_m \) using \( \nabla_{\theta_m^k} (L_m)_c \) and \( \nabla_{\theta_m^k} (L_s) \).
    \ENDFOR
\ENDFOR
\end{algorithmic}
\end{algorithm}

\subsection{Inference: Root Cause Analysis}
\begin{algorithm}[H]
\caption{Inference: Federated Root Cause Analysis}
\label{alg:inferencing}
\begin{algorithmic}[1]
\STATE \textbf{Input:} Residuals \( r_c^t, r_a^t \), thresholds \( \tau_c, \tau_a \)
\STATE \textbf{Output:} Root cause and propagated effect; causal network
\FOR{each \textit{client }\( m \) in parallel}
    \STATE Compute Mahalanobis distances \( d_{M,c}^2, d_{M,a}^2 \).
    \STATE Generate anomaly flags: \( Z_c^t = \mathbb{1}[d_{M,c}^2 > \tau_c], Z_a^t = \mathbb{1}[d_{M,a}^2 > \tau_a] \) where $\mathbb{1}$ is indicator function.
    \STATE Transmit \((Z_c^t, Z_a^t)\) to the server.
\ENDFOR

\STATE \textit{Server:} Use look-up Table 1 to identify clients with root cause and propagated effects of anomaly
\end{algorithmic}
\end{algorithm}